\documentclass[accepted]{uai2026}

\usepackage[american]{babel}

\usepackage{natbib} 
\usepackage{mathtools} 
\usepackage{booktabs} 
\usepackage{tikz} 
\usepackage{amsthm}
\usepackage{amssymb}
\usepackage{bbding}
\usepackage{esvect}
\usepackage{thm-restate}
\usepackage{comment}
\usepackage{bbm}

\usetikzlibrary{arrows,patterns,decorations.pathmorphing,positioning,shapes.geometric}
\tikzstyle{var}=[circle,draw,thick,minimum size=18pt,inner sep=0pt]
\tikzstyle{varc}=[rectangle,draw,thick,minimum size=18pt,inner sep=2pt]
\tikzstyle{arr}=[->,>=stealth',draw,fill,thick]
\tikzstyle{otto}=[draw,fill,thick,double]
\tikzstyle{noarr}=[draw,fill,thick]

\newcommand{\ot}{\mathrel{\leftarrow}}

\newcommand{\ttt}{\mathrel{{\relbar\mkern-9mu\relbar}}}

\newcommand\otto{\mathrel{{\Relbar\mkern-9mu\Relbar}}}

\newcommand\acy{\mathrm{acy}}
\DeclareMathOperator*{\Indep}{\perp\mkern-11mu\perp}
\DeclareMathOperator*{\Perp}{\perp}
\DeclareMathOperator*{\nPerp}{\not\perp}
\newcommand\given{\,|\,}
\newcommand\mk{\,\Vert\,}
\newcommand\Prb{\mathbb{P}}
\newcommand\pa{\mathrm{pa}}
\newcommand\anc{\mathrm{anc}}
\newcommand\ant{\mathrm{ant}}
\newcommand\doit{\mathrm{do}}
\newcommand\fdet{\mathrm{fdet}}
\newcommand\sm{\setminus}

\newcommand{\ind}{\mathbbm{1}}

\newtheorem{theorem}{Theorem}
\newtheorem{proposition}[theorem]{Proposition}
\newtheorem{lemma}[theorem]{Lemma}
\newtheorem{corollary}[theorem]{Corollary}
\newtheorem{definition}[theorem]{Definition}
\newtheorem{assumption}[theorem]{Assumption}
\newtheorem{example}[theorem]{Example}

\title{Causal Reasoning with Bipartite Graphical Causal Models}

\author[1]{\href{mailto:<j.m.mooij@uva.nl>}{Joris M.~Mooij}}
\affil[1]{%
    Korteweg-De Vries Institute for Mathematics\\
    University of Amsterdam\\
    Amsterdam, the Netherlands
}

\begin{document}
\maketitle

\begin{abstract}%
Causal Bayesian networks (CBNs) and structural causal models (SCMs) are the
dominant frameworks for graphical causal reasoning, but they cannot adequately represent
all real-world causal systems. In particular, systems at equilibrium---where
feedback mechanisms create cyclic causal dependencies---can exhibit causal
semantics that are fundamentally incompatible with these frameworks: different
interventions that enforce the same variable value may have different effects,
rendering the standard ``perfect intervention'' $\doit(X\!=\!x)$ ambiguous.
We propose \emph{bipartite graphical causal models} (BGCMs), in which the
structure of a system of equations is encoded by a bipartite graph with
variable and equation nodes. In this framework, a hard intervention
$\doit(f_j : X_v\!=\!\xi_v)$ specifies which equation is replaced, which
variable is targeted, and at what value---resolving the ambiguity of the
standard notion. We demonstrate, through a detailed case study of a physical
system, that this representation naturally corresponds to distinct
real-world interventions. We formulate a Markov property in terms of a new
graphical separation criterion ($B$-separation) that exploits the functional
determinism inherent in the equations, and we extend it to settings with
non-random inputs. We show how this gives rise to
a do-calculus for reasoning about domain invariances. BGCMs strictly generalize
CBNs and SCMs while retaining the ability to perform graphical causal reasoning.
\end{abstract}

\section{Introduction}\label{sec:intro}

In many scientific disciplines---physics, engineering, economics, biology---complex systems are naturally described by systems of equations relating endogenous and exogenous variables. Each equation represents an independent mechanism or physical law; the exogenous variables represent external inputs or noise. Causal Bayesian networks (CBNs) \citep{Pearl2009} and structural causal models (SCMs) \citep{Pearl2009,Bongers++_AOS_21} can represent many such systems, but not all. In particular, systems at equilibrium---where feedback mechanisms create cyclic causal dependencies---can exhibit causal semantics that are fundamentally incompatible with these frameworks.

A canonical example, due to \citet{IwasakiSimon1994}, is a bathtub at equilibrium: the equilibrium relations between inflow, outflow, pressure, and depth form a system of equations whose causal interpretation depends on \emph{which equation} is changed by an intervention. Different interventions that set the same variable to the same value can have different effects on other variables, rendering the standard notion of a ``perfect intervention'' $\doit(X = x)$ ambiguous \citep{BlomVanDiepenMooij_JMLR_21}. Neither CBNs nor SCMs can express this distinction, since they associate each variable with a unique structural equation or Markov kernel.

The key observation underlying our approach is that the structure of a system of equations is naturally encoded by a \emph{bipartite graph}, with two types of nodes---variable nodes and equation nodes---connected by an edge whenever a variable appears in an equation. By retaining both variable and equation nodes as first-class citizens, the bipartite graph preserves information that is lost when projecting onto a directed graph over variables alone. 
Building on Simon's causal ordering algorithm \citep{Simon1953}, which derives a partial causal ordering of variables by analyzing the bipartite graph, we develop a full-fledged causal modeling framework. This approach is closely related to how engineers already reason about causality, for example in the equation-based modeling language Modelica, where systems are specified as sets of ``acausal'' equations and causality is derived automatically through symbolic analysis \citep{BunusFritzson2002}.

In particular, the bipartite representation makes it possible to represent a hard intervention as $\doit(f_j : X_v = \xi_v)$---specifying which equation $f_j$ is replaced, which variable $X_v$ is targeted, and at what value $\xi_v$. While this intervention notion was already proposed by \citet{BlomVanDiepenMooij_JMLR_21}, it appeared ad hoc, and it remained unclear to what extent it provides a proper and natural mathematical abstraction of real-world interventions.


The contributions of this paper are:
\begin{enumerate}
    \item We formally define \emph{bipartite graphical causal models} (BGCMs) and show how they strictly extend CBNs, acyclic SCMs, simple SCMs, and general SCMs.
    \item We demonstrate, through a complete analysis of all hard interventions on the bathtub system, that the BGCM intervention notion $\doit(f_j : X_v = \xi_v)$ provides a \emph{natural and physically meaningful} representation of real-world interventions---each corresponding to a distinct physical procedure with distinct causal effects.
    \item We formulate a \emph{Markov property} for BGCMs using a graphical separation criterion ($B$-separation) that generalizes $d$-separation to partially oriented bipartite graphs. By encoding both the causal structure and the conditional independence structure in a single partially oriented bipartite graph---rather than in separate graphs as in \citet{BlomVanDiepenMooij_JMLR_21}---we make the connection between causal and Markov semantics transparent. Because $B$-separation exploits the functional determinism induced by the equations, the resulting Markov property is strictly stronger than the one obtainable from the Markov ordering graph of \citet{BlomVanDiepenMooij_JMLR_21} via $d$-separation.
    \item We establish an \emph{extended} Markov property for the case where some exogenous variables are treated as non-random inputs, phrased in terms of transitional conditional independence \citep{Forre2021}.
    \item We develop a \emph{do-calculus} for BGCMs by exploiting this connection: the Markov property yields domain invariances---relationships between observational and interventional distributions---that go beyond Pearl's three rules \citep{Pearl2009}.
\end{enumerate}

\section{Background}\label{sec:background}

\subsection{Modeling Cyclic Causal Relations}\label{sec:cycles}

Feedback mechanisms in dynamical systems may induce cyclic causal relationships at equilibrium. Fast dynamical interactions can lead to effectively ``instantaneous'' causal cycles. Examples arise across many disciplines: coupled oscillators in physics, supply--demand--price feedback loops in economics, gene regulatory networks in biology, and climate feedback mechanisms. In many such applications, the ability to model causal cycles is essential, and acyclic models are insufficient.


We briefly review some existing causal modeling frameworks and their relationships. A \emph{causal Bayesian network} (CBN) consists of a directed acyclic graph (DAG) together with a collection of Markov kernels, one for each variable given its parents in the DAG \citep{Pearl2009}. An \emph{acyclic structural causal model} (acyclic SCM) consists of a set of structural equations of the form $X_i := f_i(\pa(X_i), U_i)$, where $\pa(X_i)$ denotes the parents of $X_i$ and $U_i$ is an exogenous noise variable, together with an acyclic causal graph \citep{Pearl2009}. An \emph{SCM} extends this to allow cyclic causal graphs, with each equation having a unique ``dependent variable'' \citep{Bongers++_AOS_21}. While general SCMs can be complicated to work with, the subclass of \emph{simple SCMs} allows for (sufficiently weak) cycles and retains most of the convenient mathematical properties of acyclic SCMs \citep{Bongers++_AOS_21}.

These frameworks form a hierarchy of increasing generality:
\begin{equation*}\begin{split}
  \text{CBNs} & \subset \text{acyclic SCMs} \subset \text{simple SCMs} \\
  & \subset \text{SCMs} \subset \text{BGCMs} \subset \text{CCMs},
\end{split}\end{equation*}
where BGCMs are bipartite graphical causal models (introduced in this paper) and CCMs are causal constraint models \citep{BlomBongersMooij_UAI_19}. 
BGCMs occupy a position in this hierarchy that balances model flexibility with the ability to perform causal reasoning.


\citet{Simon1953} introduced the \emph{causal ordering algorithm}, which derives a causal interpretation of a system of equations from its structural properties. The key idea is that given a system of equations with designated exogenous variables, one can determine a partial ordering on the endogenous variables by analyzing which subsets of equations can be solved for which subsets of variables, and in what order.
The original version of the algorithm as proposed by \citet{Simon1953} required solving NP-hard subproblems. 
Later, \citet{Nayak1995} proposed a computationally efficient version based on perfect matchings.

\citet{BlomVanDiepenMooij_JMLR_21} expanded upon Simon's approach to causality by using his algorithm to construct two different graphs out of a set of equations: the \emph{causal ordering graph} (a directed cluster graph that represents causal effects of certain interventions) and the \emph{Markov ordering graph} (a directed graph on variable nodes, obtained by declustering and marginalizing out equation nodes, from which conditional independences can be read off via $d$-separation).

\section{Bipartite Graphical Causal Models}\label{sec:bgcm}

\subsection{Systems of Equations and Bipartite Graphs}\label{sec:bipartite}

We consider a system of equations involving a set of variables $X_V = (X_v)_{v \in V}$ taking values in standard Borel spaces $(\mathcal{X}_v)_{v \in V}$, and partitioned into \emph{endogenous} variables $X_{V \sm U}$ and \emph{exogenous} variables $X_U$, for some $U \subseteq V$. The equations $\{f_j\}_{j \in F}$ are of the form $0 = \phi_j(X_{\mathrm{nb}(f_j)})$, where $\phi_j$ is a measurable function and $\mathrm{nb}(f_j) \subseteq V$ denotes the set of variables appearing in equation $f_j$.

\begin{definition}[Bipartite graph of a system of equations]\label{def:bipartite-graph}
The \emph{bipartite graph} of a system of equations is the undirected bipartite graph $G = (V, F, E)$, where $V$ is the set of variable nodes, $F$ is the set of equation nodes, and $E \subseteq V \times F$ contains an edge $(v, f)$ if and only if variable $X_v$ appears in equation $f$.
\end{definition}

We illustrate this with a running example, a simplification of the bathtub system in \citep{IwasakiSimon1994}.

\begin{figure}
  \centering
  \includegraphics[width=0.45\textwidth]{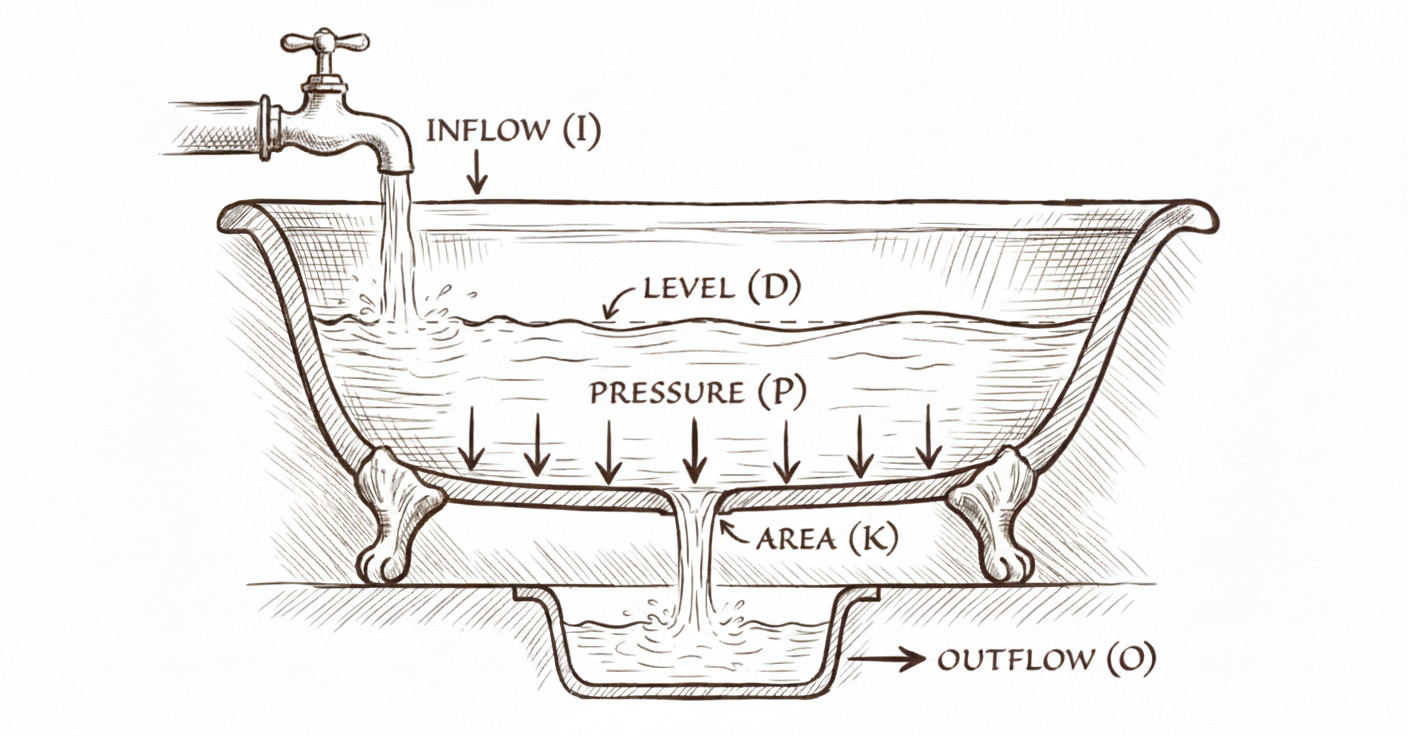}
  \caption{The bathtub system (rendered by Google Gemini).\label{fig:bathtub}}
\end{figure}
\begin{example}[Bathtub at equilibrium]\label{ex:bathtub}
Consider a bathtub with constant inflow of water at equilibrium (Figure~\ref{fig:bathtub}). The endogenous variables are: $X_O$ (water outflow through the drain), $X_D$ (water depth), and $X_P$ (pressure at the drain). The exogenous variables are: $X_I$ (water inflow from faucet), $X_K$ (drain area), and $X_g$ (gravitational acceleration). The equilibrium is described by three independent mechanisms:
\begin{align}
    f_1: & \quad 0 = X_I - X_O \label{eq:f1}\\
    f_2: & \quad 0 = X_K \sqrt{X_P} - X_O \label{eq:f2}\\
    f_3: & \quad 0 = X_g X_D - X_P \label{eq:f3}
\end{align}
Equation~\eqref{eq:f1} states that at equilibrium, outflow equals inflow. Equation~\eqref{eq:f2} is Torricelli's law: outflow is proportional to the drain area and the square root of the pressure. Equation~\eqref{eq:f3} is Stevin's law: pressure is proportional to depth and gravitational acceleration.

The bipartite graph of this system has variable nodes $\{X_O, X_P, X_D, X_I, X_K, X_g\}$ and equation nodes $\{f_1, f_2, f_3\}$, with edges connecting each equation to the variables appearing in it. We use squares for equation nodes, circles for endogenous nodes, while exogenous variables $\{X_I, X_K, X_g\}$ are shown without circles:

\medskip
\centerline{\scalebox{0.8}{\begin{tikzpicture}
\node[var] (vO) at (2,0) {$X_O$};
\node[var] (vP) at (4,0) {$X_P$};
\node[var] (vD) at (6,0) {$X_D$};
\node[varc] (f1) at (2,1.5) {$f_1$};
\node[varc] (f2) at (4,1.5) {$f_2$};
\node[varc] (f3) at (6,1.5) {$f_3$};
\node (XI) at (2,3) {$X_I$};
\node (XK) at (4,3) {$X_K$};
\node (Xg) at (6,3) {$X_g$};
\draw[noarr] (f1) -- (vO);
\draw[noarr] (f2) -- (vP);
\draw[noarr] (f2) -- (vO);
\draw[noarr] (f3) -- (vD);
\draw[noarr] (f3) -- (vP);
\draw[noarr] (f2) -- (XK);
\draw[noarr] (f1) -- (XI);
\draw[noarr] (f3) -- (Xg);
\end{tikzpicture}}}
\end{example}

\subsection{Causal Ordering and Partial Orientation}\label{sec:causal-ordering}

Given a bipartite graph $G = (V, F, E)$ and a set $U \subseteq V$ of exogenous variables, Simon's causal ordering algorithm produces a partial orientation of $G$ that encodes the causal structure.

The algorithm first finds a \emph{perfect matching} $M$ of the subgraph $G_{(V \sm U) \cup F}$, i.e., a subset of edges such that each endogenous variable node and each equation node is incident to exactly one edge in $M$.\footnote{For simplicity of exposition, we assume throughout that there is such a perfect matching. If this is not the case, one can make use of the Dulmage-Mendelsohn decomposition \citep{BlomVanDiepenMooij_JMLR_21}.} This matching associates each equation with a unique endogenous variable, which can be thought of as the variable that the equation ``solves for.'' \citet{Nayak1995} showed that the perfect matching can be found efficiently using the Hopcroft-Karp algorithm \citep{HopcroftKarp1973}.

We then define an equivalence relation on the nodes of $G$ that identifies nodes belonging to the same cluster.

\begin{definition}[Equivalence relation and clusters]\label{def:equivalence}
  Given a bipartite graph $G = (V, F, E)$, subset $U \subseteq V$, and perfect matching $M$ of $G_{(V \sm U) \cup F}$, define $\sim$ as the equivalence relation on $V \cup F$ generated by the following: $a \sim b$ if $\text{$a \ttt b$} \in M$, or if $a$ and $b$ lie on a closed $M$-alternating walk (i.e., a walk that alternates between matched and unmatched edges and returns to its starting node). The equivalence class of a node $a$ is denoted $[a]$, and we refer to this as a ``cluster''.
\end{definition}

\begin{lemma}[\citealp{DulmageMendelsohn1958}]\label{lem:DM}
The equivalence relation $\sim$ depends only on the bipartite graph $G$ and the set of exogenous variables $U$, not on the choice of perfect matching $M$.
\end{lemma}

Note that each exogenous node forms a singleton cluster.
Using this equivalence relation, we define the partial orientation of the bipartite graph.

\begin{definition}[Partial orientation]\label{def:partial-orientation}
The \emph{partial orientation} $\vv{G}$ of $G$ is obtained by orienting each edge $v \ttt f \in E$ (with $v \in V$, $f \in F$) as follows:
\[
    v \ttt f \mapsto \begin{cases}
        v \to f & \text{if } v \not\sim f, \\
        v \otto f & \text{if } v \sim f.
    \end{cases}
\]
\end{definition}

The mapping $G \mapsto \vv{G}$ is equivalent to Simon's causal ordering algorithm \citep{Simon1953}. Directed edges $v \to f$ indicate that variable $v$ is an input to equation $f$ from a different (earlier) cluster. Double-undirected edges $v \otto f$ indicate that $v$ and $f$ belong to the same cluster. Note that since matched pairs always satisfy $v \sim f$, matched edges are always oriented as $v \otto f$.\footnote{In Simon's original formulation, matched edges are directed as $f \to v$, suggesting that $f$ ``solves for'' $v$. We use $v \otto f$ uniformly because solvability is not guaranteed by the graph structure alone, but is an additional assumption (Assumption~\ref{ass:leu}).}

\begin{example}[Bathtub: causal ordering]\label{ex:bathtub-ordering}
For the bathtub system (Example~\ref{ex:bathtub}), the (unique) perfect matching associates $f_1$ with $X_O$, $f_2$ with $X_P$, and $f_3$ with $X_D$. The partially oriented graph is:

\medskip
\centerline{\scalebox{0.8}{\begin{tikzpicture}
\node[var] (vO) at (2,0) {$X_O$};
\node[var] (vP) at (4,0) {$X_P$};
\node[var] (vD) at (6,0) {$X_D$};
\node[varc] (f1) at (2,1.5) {$f_1$};
\node[varc] (f2) at (4,1.5) {$f_2$};
\node[varc] (f3) at (6,1.5) {$f_3$};
\node (XI) at (2,3) {$X_I$};
\node (XK) at (4,3) {$X_K$};
\node (Xg) at (6,3) {$X_g$};
\draw[otto] (f1) -- (vO);
\draw[otto] (f2) -- (vP);
\draw[arr] (vO) to (f2);
\draw[otto] (f3) -- (vD);
\draw[arr] (vP) to (f3);
\draw[arr] (XK) to (f2);
\draw[arr] (XI) to (f1);
\draw[arr] (Xg) to (f3);
\end{tikzpicture}}}

\medskip
\noindent This encodes the causal ordering: first solve $f_1$ for $X_O$ in terms of $X_I$, yielding $X_O = X_I$; then solve $f_2$ for $X_P$ in terms of $X_O$ and $X_K$, yielding $X_P = X_I^2 / X_K^2$; finally solve $f_3$ for $X_D$ in terms of $X_P$ and $X_g$, yielding $X_D = X_I^2 / (X_K^2 X_g)$.
The clusters are $\{X_I\},\{X_K\},\{X_g\},\{f_1,X_O\},\{f_2,X_P\},\{f_3,X_D\}$.
\end{example}

We now formally define what we mean by ``causal ordering''.
We define a \emph{walk} in $\vv{G}$ as an alternating sequence of nodes and edges $n_0, e_1, n_1, e_2, \ldots, e_k, n_k$ where each edge $e_i$ connects $n_{i-1}$ and $n_i$; $k=0$ corresponds with a trivial walk.
A \emph{path} is a walk in which no node repeats.
\begin{definition}[Anterior]\label{def:anterior}
A node $a$ is \emph{anterior} to a node $b$ in $\vv{G}$ if there is a walk from $a$ to $b$ consisting only of $\to$ and $\otto$ edges. 
That is, a walk of the form
$n_0, e_1, n_1, e_2, \ldots, e_k, n_k$
where $n_0 = a$, $n_k = b$, and each edge $e_i$ is either $n_{i-1} \to n_i$ or $n_{i-1}\otto n_i$. 
We write $\ant_{\vv{G}}(B)$ for the set of all nodes anterior to some node in $B \subseteq V \cup F$.
\end{definition}
Mathematically, the relation ``is anterior to'' is a \emph{partial order} on the clusters (it is reflexive and transitive; it is antisymmetric because any two mutually anterior nodes lie in the same cluster).
Simon's insight was that this mathematical relationship captures what we perceive as causes and their effects: we consider $a$ a cause of $b$ precisely if $a \in \ant_{\vv{G}}(b)$.


\subsection{Solutions, Distributions, and Markov Kernels}\label{sec:solutions}

By solving the system of equations according to the causal ordering, we obtain \emph{solution functions} that express all endogenous variables in terms of the exogenous variables.
When a cluster contains more than one equation, the equations in the cluster must be solved simultaneously for the variables in that cluster.

\begin{definition}[Solution function]\label{def:solution}
  A \emph{solution function} for a system of equations with exogenous variables $X_U$ and endogenous variables $X_{V \sm U}$ is a function $\Phi: \mathcal{X}_U \to \mathcal{X}_V$ such that $\Phi_U(x_U) = x_U$, and $\Phi(x_U)$ satisfies all equations for every $x_U \in \mathcal{X}_U$.
\end{definition}

If we assume that all exogenous variables are mutually independent random variables with distributions $X_u \sim \Prb(X_u)$ for $u \in U$, then the \emph{joint distribution} $\Prb(X_V)$ of all variables is obtained as the pushforward of the product distribution $\bigotimes_{u \in U} \Prb(X_u)$ through the solution function $\Phi$.
This corresponds to the distribution of $X_V$ under the sampling scheme
  \[ X_u \sim \Prb(X_u)\ \ \text{for } u \in U, \qquad X_V = \Phi(X_U). \]
More generally, we can treat some exogenous variables as random and others as non-random. 
This yields a \emph{Markov kernel} rather than a distribution. 
If one only assigns independent distributions to exogenous variables in subset $U \setminus J$ with $J \subseteq U$, one obtains the Markov kernel $\Prb(X_V \mk X_J)$, corresponding to the sampling scheme
  \[ X_u \sim \Prb(X_u)\ \ \text{for } u \in U \setminus J, \qquad X_V = \Phi(X_J,X_{U\sm J}) \]
where the \emph{input variables} $X_J$ are left unconstrained and we make no assumptions about their distribution.
For instance, treating $X_I$ as a non-random input and $X_K, X_g$ as random yields the Markov kernel $\Prb(X_K, X_g, X_O, X_P, X_D \mk X_I)$, which specifies the joint distribution of $X_K,X_g,X_O,X_P,X_D$ for every possible value $x_I$.

\begin{definition}[Unique solvability of a cluster]\label{def:cluster_uniquely_solvable}
  A cluster $[c]$ in the partially oriented graph $\vv{G}$ is called \emph{uniquely solvable} if the equations in $F \cap [c]$ can be solved for the variables $V \cap [c]$ in terms of $\pa_{\vv{G}}([c])$, and the local solution function $\Phi^{[c]} : \mathcal{X}_{\pa_{\vv{G}}([c])} \to \mathcal{X}_{[c] \cap V}$ is unique.
\end{definition}

If all local unique solvability assumptions are met, this guarantees existence and uniqueness of a global solution function.
We collect our assumptions so far:
\begin{assumption}\label{ass:leu}
For a system of equations 
  \[0 = \phi_j(X_{\mathrm{nb}(f_j)}), \quad j \in F\]
with corresponding bipartite graph $G = (V, F, E)$, exogenous variables $U \subseteq V$, standard Borel spaces $(\mathcal{X}_v)_{v\in V}$:
\begin{enumerate}
  \item The functions $\phi_j : \mathcal{X}_{\mathrm{nb}(f_j)} \to \mathbb{R}$ are measurable.
  \item The subgraph $G_{(V \sm U) \cup F}$ has a perfect matching.
  \item The exogenous variables are \emph{variation independent}: their joint value space is a Cartesian product $\prod_{u \in U} \mathcal{X}_u$.
  \item The system is \emph{clusterwise uniquely solvable}: each endogenous cluster $[v]$ (for $v \in V \sm U$) is uniquely solvable.
\end{enumerate}
\end{assumption}

\begin{proposition}\label{prop:global-solution}
  Under Assumption~\ref{ass:leu}, there exists a unique solution function $\Phi : \mathcal{X}_U \to \mathcal{X}_V$.
\end{proposition}
\begin{proof}
The local solution functions $(\Phi^{[v]})_{v \in V\sm U}$ combine into a system of equations
  \begin{align*}
    X_v = \Phi_v^{[v]}(X_{\pa_{\vv{G}}([v])}) \qquad v \in V
  \end{align*}
that has an acyclic structure. 
Recursive substitution of the local solution functions of each cluster along the causal ordering then yields the global solution function.
\end{proof}
This unique (global) solution function then induces a unique joint distribution $\Prb(X_V)$ and unique Markov kernels $\Prb(X_V \mk X_J)$ for $J \subseteq U$.

\section{Markov Property}\label{sec:Markov}

In this section, we formulate a Markov property for bipartite graphical causal models that relates the conditional independence structure of the joint distribution to the graphical structure of the partially oriented bipartite graph.

\subsection{$B$-Separation}\label{sec:Bsep}

We define a graphical separation criterion for partially oriented bipartite graphs, called \emph{$B$-separation} (for ``bipartite''), an analog of classical $d$-separation that is appropriate for our setting.
It combines ideas from the segment-based formulation of $\sigma$-separation \citep{ForreMooij_1710.08775} to deal with cycles (clusters with more than a single variable), and from $D$-separation \citep{GeigerVermaPearl1990} to take into account deterministic relations (each endogenous cluster is a deterministic function of its parents).

\begin{definition}[Segments, exits, collider/non-collider segments]
  A walk $s$ on $\vv{G}$ can be partitioned into \emph{segments}: consecutive maximal subwalks $s_1, \dots, s_m$ of the form $s_{i,1} \otto \dots \otto s_{i,k_i}$ (possibly $k_i = 1$), i.e., with all edges double-undirected ($\otto$).
  We call a boundary node of a segment an \emph{exit} if its bounding edge points out of the segment or it is an end node of the walk.
  That is, $s_{i,1}$ is an exit if the edge on $s$ to the left of it is $\ldots \ot s_{i,1}$ or if it is the first node of $s$ ($i=1$), and $s_{i,k_i}$ is an exit if the edge on $s$ to the right of it is $s_{i,k_i} \to \ldots$ or if it is the last node of $s$ ($i=m$).
  A segment with no exit is a \emph{collider segment}; a segment with one or two exits is a \emph{non-collider segment}.
\end{definition}

The following notion tracks deterministic relations that are imposed by the structure of the bipartite graph.
\begin{definition}[Functionally determined]\label{def:functionally_determined}
  Let $C \subseteq V$ be a subset of variable nodes in $\vv{G}$. 
  Define $C_0 := C$ and
  \[ C_{n+1} := C_n \cup \{ v \in V \sm U : \pa_{\vv{G}}([v]) \subseteq C_n \}. \]
  We define $\fdet_{\vv{G}}(C) := \bigcup_{n \ge 0} C_n$ and refer to those as the variable nodes that are \emph{functionally determined by $C$}.
\end{definition}
Note that exogenous variable nodes are only functionally determined by $C$ if they are in $C$.

With these definitions in place, we define:
\begin{definition}[$B$-blocking]\label{def:B-blocking}
  For $C \subseteq V$, the walk is called \emph{$B$-blocked by $C$} if it contains:
  \begin{enumerate}
    \item a \emph{collider segment} that does not intersect $\ant_{\vv{G}}(C)$, or
    \item a \emph{non-collider segment} with an exit in $\fdet_{\vv{G}}(C)$.
  \end{enumerate}
  Otherwise, the walk is called \emph{$B$-open given $C$}.
\end{definition}
\begin{definition}[$B$-separation]\label{def:B-sep}
Let $A, B, C \subseteq V$ be sets of variable nodes. We say that \emph{$A$ and $B$ are $B$-separated given $C$ in $\vv{G}$}, written
\[
A \Perp^{B}_{\vv{G}} B \given C,
\]
if every walk from a node in $A$ to a node in $B$ is $B$-blocked by $C$ in $\vv{G}$.\footnote{It is equivalent to require only that every \emph{path} from $A$ to $B$ be $B$-blocked by $C$ (Lemma~\ref{lem:B-sep_walk_path}); the path formulation is usually more convenient to check by hand.}
\end{definition}
To build intuition, it is helpful to see how $B$-separation adapts $d$-separation to the two features that distinguish partially oriented bipartite graphs from DAGs: clusters and determinism.

\emph{Clusters.} The variables and equations of a cluster are solved jointly, so segments along a walk behave as a single indivisible unit. A walk can enter or leave a segment only through an \emph{exit}, and conditioning therefore interacts with a segment only through its exits. This is the bipartite-graph counterpart of collapsing a strongly connected component in $\sigma$-separation \citep{ForreMooij_1710.08775}: whether a segment blocks or transmits dependence is decided by the segment as a whole rather than node by node. As in $d$-separation, a \emph{non-collider} segment (a chain, a fork, or an endpoint of the walk) transmits dependence unless it is ``pinned down'' by the conditioning set, whereas a \emph{collider} segment (a common effect, entered by arrows from both sides) blocks association unless it is ``activated'' by the conditioning set.

\emph{Determinism.} Since each endogenous cluster is a deterministic function of its parents, conditioning on $C$ fixes not only $X_C$ but every variable that is functionally determined by it, i.e., $X_{\fdet_{\vv{G}}(C)}$. This is what ``pinned down'' means here: a non-collider segment is already blocked once one of its exits lies in $\fdet_{\vv{G}}(C)$, even if that exit is not itself in $C$. Activation of colliders, on the other hand, is governed by $\ant_{\vv{G}}(C)$: a collider segment transmits dependence only if the common effect, or one of its descendants, actually belongs to the conditioning set $C$. It is precisely this extra blocking granted by $\fdet_{\vv{G}}(C)$ that makes $B$-separation stronger than criteria that ignore determinism. This intuition is made precise in Appendix~\ref{app:direct-proof}.

\subsection{Global Markov Property}\label{sec:global-Markov}

\begin{restatable}[Global Markov property]{theorem}{thmMarkovB}\label{thm:Markov_B}
  Suppose that Assumption~\ref{ass:leu} holds.
  When assigning independent distributions to all exogenous variables, the resulting joint distribution $\Prb(X_V)$ satisfies: for all $A, B, C \subseteq V$, 
  \[ A \Perp^B_{\vv{G}} B \given C \implies X_A \Indep_{\Prb(X_V)} X_B \given X_C. \]
\end{restatable}

The Markov property ``propagates'' the independence of the exogenous variables through the equations along the partial ordering, yielding conditional independences among endogenous variables.

\begin{example}[Bathtub: Markov property]\label{ex:bathtub-Markov}
In the partially oriented bathtub graph (Example~\ref{ex:bathtub-ordering}), every path from $X_D$ to $X_O$ must pass through $X_P$ (via the equation nodes). One can verify that $X_D \Perp^{B}_{\vv{G}} X_O \given X_P$, which implies the conditional independence:
\[
X_D \Indep X_O \given X_P.
\]
This means the joint distribution factorizes as
\[
\Prb(X_D, X_O, X_P) = \Prb(X_D \given X_P) \otimes \Prb(X_O, X_P).
\]
\end{example}

\subsection{Extended Global Markov Property}\label{sec:extended-Markov}

A more general version of the Markov property allows treating some exogenous variables as non-random, using an extended notion of conditional independence \citep{Forre2021}.

\begin{restatable}[Extended Global Markov property]{theorem}{thmiMarkovB}\label{thm:i-Markov_B}
  Suppose Assumption~\ref{ass:leu} holds.
  Treat exogenous variables $J \subseteq U$ as non-random, and assign independent distributions to exogenous variables in $U \sm J$, yielding Markov kernel $\Prb(X_V \mk X_J)$.
  Then for all $A, B, C \subseteq V$ such that $J \subseteq B \cup C$:\footnote{One can replace the assumption $A \Perp^{B}_{\vv{G}} B \given C$ by $A \Perp^{B}_{\vv{G}} B \cup J \given C$ to obtain a result that is valid for \emph{all} choices of $A,B,C$.}
  \[ A \Perp^{B}_{\vv{G}} B \given C \implies X_A \Indep_{\Prb(X_V \mk X_J)} X_B \given X_C. \]
\end{restatable}
Concretely, the conditional independence $X_A \Indep_{\Prb(X_V \mk X_J)} X_B \given X_C$ for a Markov kernel $\Prb(X_V \mk X_J)$ means that there \emph{exists} a Markov kernel $Q(X_A \mk X_C)$ (not depending on $X_B$) such that
\[\Prb(X_A, X_B, X_C \lVert X_J) = Q(X_A \lVert X_C) \otimes \Prb(X_B, X_C \lVert X_J).\]
This is the notion of \emph{transitional conditional independence} introduced by \citet{Forre2021}; it is asymmetric (the roles of $A$ and $B$ are not interchangeable).\footnote{Our notion of $B$-separation does not distinguish non-random and random variables explicitly, which is why the condition $J \subseteq B \cup C$ is needed here; see Appendix~\ref{app:proof-extended}.}

\begin{example}[Bathtub: extended Markov property]\label{ex:bathtub-extended-Markov}
In the bathtub model, treating $X_I$ as non-random yields the Markov kernel $\Prb(X_K, X_g, X_O, X_P, X_D \mk X_I)$. Since $X_D \Perp^{B}_{\vv{G}} X_I \given X_P$, the extended Markov property implies $X_D \Indep X_I \given X_P$, which means there exists a Markov kernel $\Prb(X_D \mk X_P)$ such that:
\[
\Prb(X_D, X_P \mk X_I) = \Prb(X_D \given X_P) \otimes \Prb(X_P \mk X_I).
\]
\end{example}

\section{Interventions in Bipartite Graphical Causal Models}\label{sec:interventions}

Causality is fundamentally about change: how does a system react to externally imposed modifications? In the BGCM framework, there are two types of elementary interventions:
\begin{enumerate}
  \item \textbf{Changing the distribution of an exogenous variable:} replacing $\Prb(X_u)$ by another distribution $\tilde{\Prb}(X_u)$.
  \item \textbf{Replacing an equation:} substituting equation $f_j$ by a different equation $\tilde{f}_j$.
\end{enumerate}

A particularly important class of the second type is the \emph{hard intervention} $\doit(f_j : X_v = \xi_v)$, which replaces equation $f_j$ by the equation $0 = X_v - \xi_v$,\footnote{For $\mathcal{X}_v \ne \mathbb{R}$ one can more generally take $0 = \ind_{\{\xi_v\}}(X_v) - 1$.} thereby fixing variable $X_v$ to value $\xi_v$ via the intervened mechanism $f_j$. This notion naturally corresponds to concrete physical procedures: different choices of $f_j$ lead to genuinely different real-world implementations, even when targeting the same variable $X_v$ at the same value $\xi_v$ (see Example~\ref{ex:bathtub-ambiguous} and Appendix~\ref{app:implementations} for detailed examples).

\subsection{Ambiguity of Perfect Interventions}\label{sec:ambiguity}

A key insight of the BGCM framework is that the standard notion $\doit(X_v = \xi_v)$ for a ``perfect intervention'' can be ambiguous: different equations can be replaced to achieve $X_v = \xi_v$, leading to different causal effects on other variables \citep{BlomVanDiepenMooij_JMLR_21}.

\begin{figure*}[t]
  \centering
  \begin{tikzpicture}
    \node at (0,-2) {$\doit(f_1: X_D=\xi_D)$};
    \node at (0,0) {\includegraphics[height=3cm]{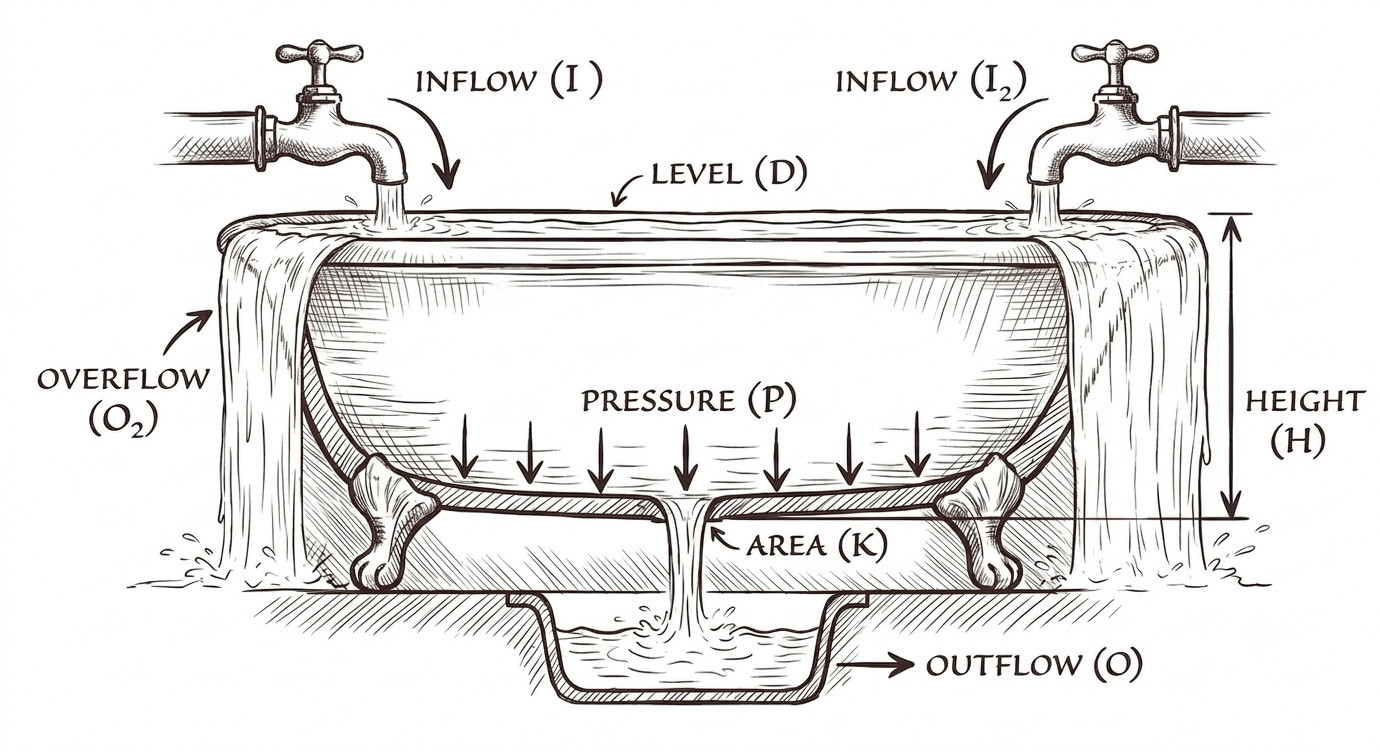}};
    \node at (5.5,-2) {$\doit(f_2: X_D=\xi_D)$};
    \node at (5.5,0) {\includegraphics[height=3cm]{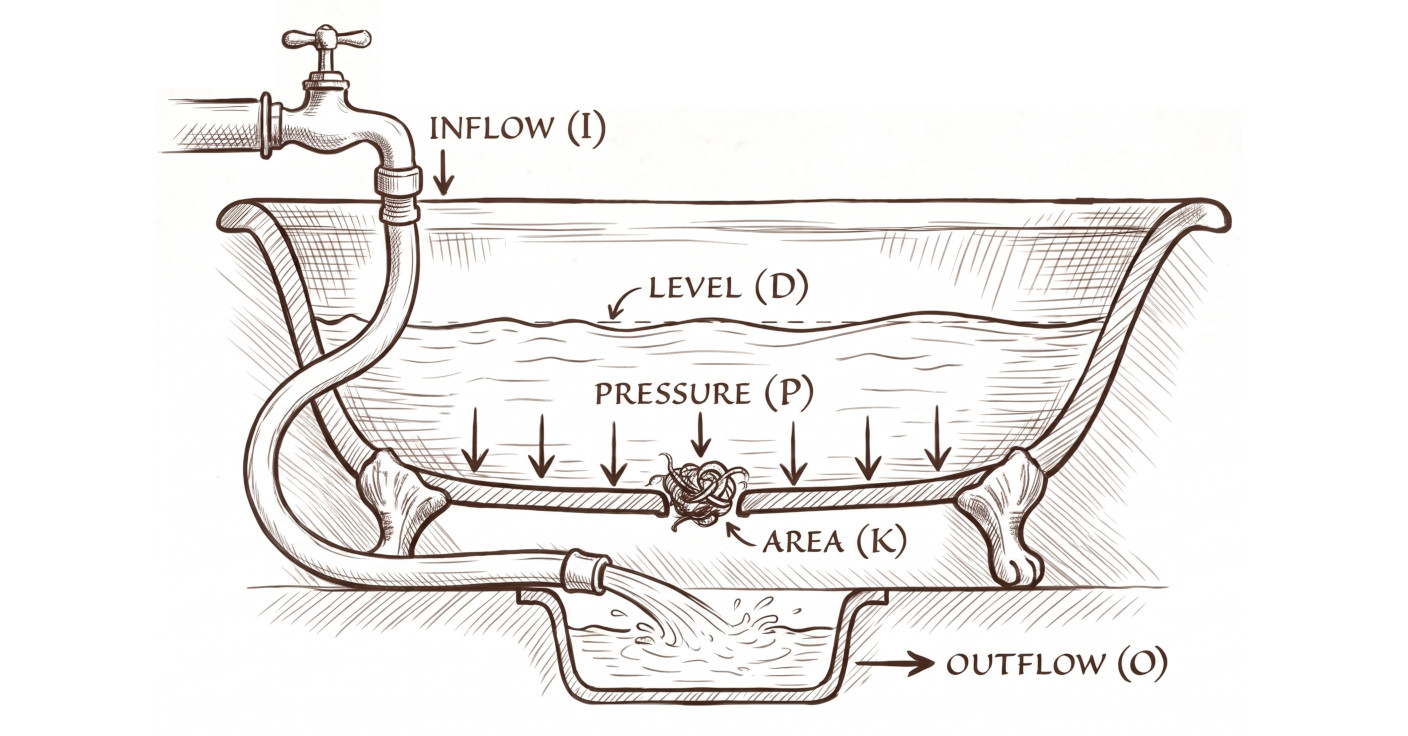}};
    \node at (11,-2) {$\doit(f_3: X_D=\xi_D)$};
    \node at (11,0) {\includegraphics[height=3cm]{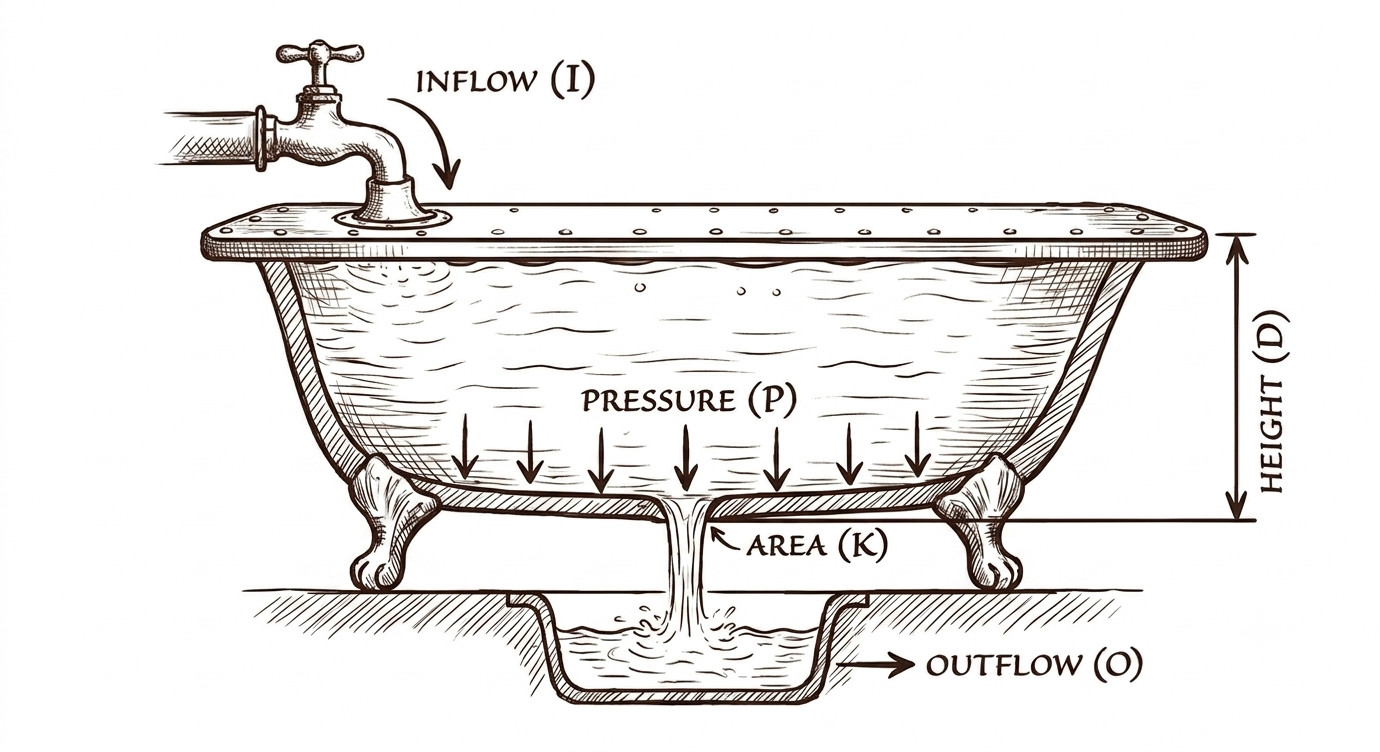}};
  \end{tikzpicture}
  \caption{Three different interventions on the bathtub that all set the water depth to a fixed value (Example~\ref{ex:bathtub-ambiguous}).\label{fig:bathtub-ambiguous}}
\end{figure*}
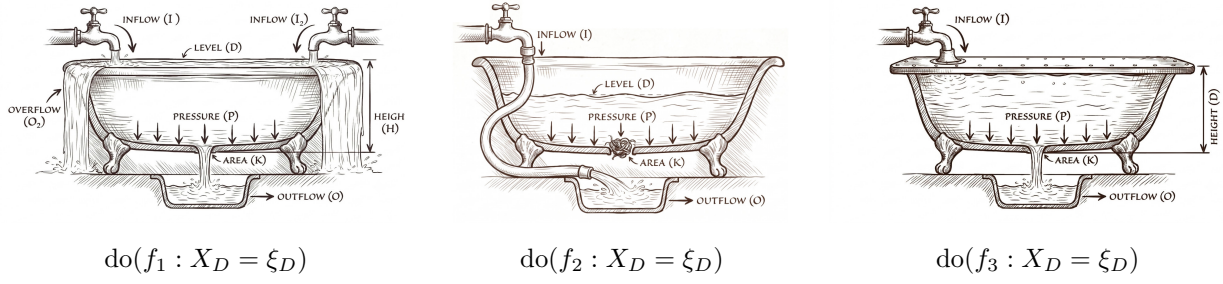

\begin{example}[Bathtub: ambiguous interventions]\label{ex:bathtub-ambiguous}
  Consider setting the water depth to a fixed value $\xi_D$ in the bathtub model. There are at least three distinct interventions that achieve this (see Figure~\ref{fig:bathtub-ambiguous}). The partially oriented bipartite graphs for these and all other hard interventions are shown in Figure~\ref{fig:intervention-graphs} in Appendix~\ref{app:intervention-graphs}.

\noindent\textbf{(i) $\doit(f_1 : X_D = \xi_D)$:} Replace the equilibrium condition~\eqref{eq:f1} by $0 = X_D - \xi_D$. Physically, this may correspond to cutting the bathtub at height $\xi_D$ and ensuring it overflows. The causal ordering \emph{reverses}: $\tilde{f}_1$ determines $X_D$, then $f_3$ determines $X_P$ from $X_D$ and $X_g$, and finally $f_2$ determines $X_O$ from $X_P$ and $X_K$. The solution is $X_O = X_K \sqrt{X_g \xi_D}$, $X_P = X_g \xi_D$, $X_D = \xi_D$.

\noindent\textbf{(ii) $\doit(f_2 : X_D = \xi_D)$:} Replace Torricelli's law~\eqref{eq:f2} by $0 = X_D - \xi_D$. Physically, this could involve disabling the drain and rerouting the inflow to the outflow once the water has reached height $\xi_D$. The causal ordering changes: $f_1$ determines $X_O$, $\tilde{f}_2$ determines $X_D$, and $f_3$ now determines $X_P$ from $X_D$. The solution is $X_O = X_I$, $X_P = X_g \xi_D$, $X_D = \xi_D$.

\noindent\textbf{(iii) $\doit(f_3 : X_D = \xi_D)$:} Replace Stevin's law~\eqref{eq:f3} by $0 = X_D - \xi_D$. Physically, this can be achieved by sealing the bathtub at height $\xi_D$ and ensuring it is completely filled. The causal ordering is preserved: $f_1$ still determines $X_O$, $f_2$ still determines $X_P$, and $\tilde{f}_3$ determines $X_D$. The solution is $X_O = X_I$, $X_P = X_I^2 / X_K^2$, $X_D = \xi_D$.

These three interventions all set $X_D = \xi_D$ but have different effects on $X_O$ and $X_P$. Therefore, the notion $\doit(X_D = \xi_D)$ is ambiguous and must be refined to $\doit(f_j : X_D = \xi_D)$. 
\end{example}

\subsection{Hard Interventions for the Bathtub}\label{sec:all-interventions}

Table~\ref{tab:interventions} summarizes all possible hard interventions for the bathtub model. 
In Appendix~\ref{app:implementations} we propose possible physical implementations of all feasible hard interventions.
That analysis shows that this is more than a purely mathematical exercise.
Not every combination of target equation and target variable yields a uniquely solvable system (signaled by the corresponding intervened graph having no perfect matching); indeed, some interventions (marked with $\emptyset$) generically lead to systems with no solution.
In those cases it would be futile to attempt to implement such interventions.

\begin{table}[t]
\caption{Feasibility of hard interventions $\doit(f_j : X_v = \xi_v)$ for the bathtub. Checkmarks indicate uniquely solvable systems; $\emptyset$ indicates the system is not uniquely solvable.}
\label{tab:interventions}
\centering
\begin{tabular}{@{}l ccc@{}}
\toprule
$\doit(f_j : X_v = \xi_v)$ & $f_1$ & $f_2$ & $f_3$ \\
\midrule
$X_O = \xi_O$ & $\checkmark$ & $\emptyset$  & $\emptyset$ \\
$X_P = \xi_P$ & $\checkmark$ & $\checkmark$ & $\emptyset$ \\
$X_D = \xi_D$ & $\checkmark$ & $\checkmark$ & $\checkmark$ \\
\bottomrule
\end{tabular}
\end{table}

Table~\ref{tab:effects} shows the solution functions for all well-defined hard interventions, illustrating how different interventions lead to different causal effects.

\begin{table}[t]
\caption{Solution functions for all hard interventions on the bathtub model.}
\label{tab:effects}
\centering
\small
\renewcommand{\arraystretch}{1.5}
\begin{tabular}{@{}l lll@{}}
\toprule
& $X_O$ & $X_P$ & $X_D$ \\
\midrule
observational & $X_I$ & $\frac{X_I^2}{X_K^2}$ & $\frac{X_I^2}{X_K^2 X_g}$ \\
\midrule
$\doit(X_I = \xi_I)$ & $\xi_I$ & $\frac{\xi_I^2}{X_K^2}$ & $\frac{\xi_I^2}{X_K^2 X_g}$ \\
$\doit(X_K = \xi_K)$ & $X_I$ & $\frac{X_I^2}{\xi_K^2}$ & $\frac{X_I^2}{\xi_K^2 X_g}$ \\
$\doit(X_g = \xi_g)$ & $X_I$ & $\frac{X_I^2}{X_K^2}$ & $\frac{X_I^2}{X_K^2 \xi_g}$ \\
\midrule
$\doit(f_1 : X_O = \xi_O)$ & $\xi_O$ & $\frac{\xi_O^2}{X_K^2}$ & $\frac{\xi_O^2}{X_K^2 X_g}$ \\
$\doit(f_1 : X_P = \xi_P)$ & $\sqrt{\xi_P}\, X_K$ & $\xi_P$ & $\frac{\xi_P}{X_g}$ \\
$\doit(f_1 : X_D = \xi_D)$ & $X_K \sqrt{X_g \xi_D}$ & $X_g \xi_D$ & $\xi_D$ \\
$\doit(f_2 : X_P = \xi_P)$ & $X_I$ & $\xi_P$ & $\frac{\xi_P}{X_g}$ \\
$\doit(f_2 : X_D = \xi_D)$ & $X_I$ & $X_g \xi_D$ & $\xi_D$ \\
$\doit(f_3 : X_D = \xi_D)$ & $X_I$ & $\frac{X_I^2}{X_K^2}$ & $\xi_D$ \\
\bottomrule
\end{tabular}
\end{table}

\subsection{Interventions Change the Causal Structure}\label{sec:interventions-change}

The bathtub system cannot be modeled as a CBN or an SCM, because $\doit(X_D = \xi_D)$ does not have a unique meaning. In a CBN or SCM, a perfect intervention $\doit(X_v = \xi_v)$ replaces a unique structural equation (the one with $X_v$ as its dependent variable), but in the bathtub there is no such unique association.

An important caveat is that hard interventions can change the bipartite graph and its partial orientation, and hence the conditional independence structure. For example, the intervention $\doit(f_1 : X_D = \xi_D)$ on the bathtub reverses the causal ordering entirely: instead of $X_O \to X_P \to X_D$, the ordering becomes $X_D \to X_P \to X_O$ (in terms of the directed part of the partially oriented graph), as can be seen in Figure~\ref{fig:intervention-graphs}. This is another phenomenon that has no counterpart in standard CBN or SCM frameworks.

\section{Domain Invariances}\label{sec:domain-invariances}

A central application of causal models is reasoning about what changes---and what remains invariant---across different experimental conditions or ``domains.'' In CBNs, Pearl's three rules of the do-calculus \citep{Pearl2009} formalize such invariances for observational versus interventional distributions. We now develop an analogous theory for BGCMs.

\subsection{The General Recipe}\label{sec:recipe}

To relate the distributions in two domains (e.g., an observational domain and an interventional domain), we employ the following procedure:
\begin{enumerate}
    \item \textbf{Construct the joint model:} Introduce an exogenous \emph{domain indicator} input variable $R$ and write the equations of both domains as a single system, where the equations that differ between domains depend on $R$.
    \item \textbf{Construct the bipartite graph $G^R$:} Build the bipartite graph of the joint model, which includes $R$ as an exogenous variable node connected to the equations in which $R$ occurs.
    \item \textbf{Run causal ordering:} Compute the partial orientation $\vv{G^R}$ of $G^R$.
    \item \textbf{Check solvability:} Verify that Assumption~\ref{ass:leu} holds for the joint model.
    \item \textbf{Apply the Markov property:} Use Theorem~\ref{thm:i-Markov_B} on $\vv{G^R}$ to derive conditional independences involving $R$, which translate into invariances across domains.
\end{enumerate}

Except for the solvability check (step 4), this is a purely graphical procedure.
When applying the conditional invariances (Examples~\ref{eg:example-II} and~\ref{eg:example-IV}), one must carefully handle the null sets arising from conditioning on continuous variables.
Tracking the null sets rigorously requires additional bookkeeping that we do not spell out here \citep[see][]{ForreMooij_Causality_2025}.

\subsection{Bathtub Examples}

\begin{example}[Observational vs.\ $\doit(X_g = \xi_g)$]\label{eg:example-I}
Consider comparing the observational setting (domain A) with the setting where gravitational acceleration is fixed to $\xi_g$ (domain B), for instance by ``moving the bathtubs to Mars.'' We introduce an exogenous variable $U_g$ and write the joint model:
  {\small\begin{align*}
    f_1: &\quad 0 = X_I - X_O, \\
    f_2: &\quad 0 = X_K \sqrt{X_P} - X_O, \\
    f_3: &\quad 0 = X_g X_D - X_P, \\
    f_4: &\quad 0 = X_g - \begin{cases} U_g & \text{if } R = A, \\ \xi_g & \text{if } R = B. \end{cases}
  \end{align*}}%
The bipartite graph $G^R$ has an additional equation node $f_4$ connected to $X_g$, $U_g$, and $R$. In the partial orientation $\vv{G^R}$ (see Figure~\ref{fig:domain-graphs}), the variables $X_P$ and $X_O$ are $B$-separated from $R$ (unconditionally). By the Markov property:
\begin{equation*}\begin{split}
X_P, X_O \Perp^{B}_{\vv{G^R}} R &\implies X_P, X_O \Indep R \\
&\implies \Prb_A(X_P, X_O) = \Prb_B(X_P, X_O).
\end{split}\end{equation*}
Equivalently, $\Prb(X_P, X_O) = \Prb(X_P, X_O \mk \doit(X_g = \xi_g))$. Hence, the joint distribution of pressure and outflow at equilibrium is invariant under changes in gravitational acceleration.
\end{example}

\begin{example}[Observational vs.\ $\doit(f_3 : X_D = \xi_D)$]\label{eg:example-II}
Now compare the observational setting with the intervention $\doit(f_3 : X_D = \xi_D)$ (sealing the bathtub). The joint model replaces $f_3$ by a domain-dependent equation:
  {\small\begin{align*}
    f_1: &\quad 0 = X_I - X_O, \\
    f_2: &\quad 0 = X_K \sqrt{X_P} - X_O, \\
    f_3: &\quad 0 = \begin{cases} X_g X_D - X_P & \text{if } R = A, \\ X_D - \xi_D & \text{if } R = B. \end{cases}
  \end{align*}}%
In the partial orientation $\vv{G^R}$ (Figure~\ref{fig:domain-graphs}), we have $X_O \Perp^{B}_{\vv{G^R}} R \given X_D, X_P$. By the Markov property:
\begin{equation*}\begin{split}
&\Prb_A(X_O \given X_D = \xi_D, X_P) \\
&\quad= \Prb_B(X_O \mk \doit(f_3 : X_D = \xi_D) \given X_P).
\end{split}\end{equation*}
This means that the conditional distribution of outflow given pressure is the same whether we \emph{observe} depth $\xi_D$ or \emph{intervene} to set it to $\xi_D$ by sealing the bathtub.
\end{example}

\begin{example}[Observational vs.\ $\doit(f_1 : X_D = \xi_D)$]\label{eg:example-III}
Comparing the observational setting with the intervention $\doit(f_1 : X_D = \xi_D)$ (cutting the bathtub and letting it overflow) yields a joint model where $f_1$ is domain-dependent:
  {\small\begin{align*}
    f_1: &\quad 0 = \begin{cases} X_I - X_O & \text{if } R = A, \\ X_D - \xi_D & \text{if } R = B, \end{cases} \\
    f_2: &\quad 0 = X_K \sqrt{X_P} - X_O, \\
    f_3: &\quad 0 = X_g X_D - X_P.
  \end{align*}}%
In the partial orientation $\vv{G^R}$ of this joint model (Figure~\ref{fig:domain-graphs}), all endogenous variables and equations belong to a single cluster (all edges are double-undirected). The Markov property does not yield non-trivial conditional independences involving $R$. Thus, we cannot use it to relate the observational and interventional distributions in this case---which is consistent with the fact that this intervention fundamentally changes the entire causal structure of the system.
\end{example}

\begin{example}[$\doit(f_1 : X_D = \xi_D)$ vs.\ $\doit(f_1 : X_D = \xi_D')$]\label{eg:example-IV}
While we cannot relate the observational distribution to $\doit(f_1 : X_D = \xi_D)$, we \emph{can} relate two interventional distributions with different parameter values. Consider the joint model where both domains have the same structural form but different intervention values:
  {\small\begin{align*}
    f_1: &\quad 0 = \begin{cases} X_D - \xi_D & \text{if } R = A, \\ X_D - \xi_D' & \text{if } R = B, \end{cases} \\
    f_2: &\quad 0 = X_K \sqrt{X_P} - X_O, \\
    f_3: &\quad 0 = X_g X_D - X_P.
  \end{align*}}%
Note that in both domains, the causal ordering is the same (reversed compared to the observational setting): $X_D \to X_P \to X_O$. In the partial orientation $\vv{G^R}$ (Figure~\ref{fig:domain-graphs}), we have $X_O \Perp^{B}_{\vv{G^R}} R \given X_P$. By the Markov property:
\begin{equation*}\begin{split}
&\Prb_A(X_O \mk \doit(f_1 : X_D = \xi_D) \given X_P) \\
&\quad= \Prb_B(X_O \mk \doit(f_1 : X_D = \xi_D') \given X_P).
\end{split}\end{equation*}
Hence, overflowing bathtubs yield the same conditional distribution of outflow given pressure, regardless of their height.
This conclusion might not be intuitively obvious but can easily be derived using our (mostly) graphical causal reasoning calculus.
\end{example}

\section{Discussion and Related Work}\label{sec:discussion}

The BGCM framework extends several existing causal modeling frameworks. Every CBN, acyclic SCM, simple SCM, and general SCM can be represented as a BGCM. Conversely, the bathtub example demonstrates that BGCMs can represent systems whose causal semantics are not captured by any of these frameworks.

Our approach builds on Simon's causal ordering algorithm \citep{Simon1953}, the $\sigma$-separation criterion for cyclic SCMs \citep{ForreMooij_1710.08775}, the $D$-separation criterion \citep{GeigerVermaPearl1990} and the framework of \citet{BlomVanDiepenMooij_JMLR_21}. The latter framework uses two distinct graphs that serve complementary purposes: \citet{BlomVanDiepenMooij_JMLR_21} show that their Markov ordering graph does not correctly represent causal effects of interventions, while their causal ordering graph does not directly encode conditional independences.

The notion that perfect interventions $\doit(X = x)$ can be ambiguous was identified by \citet{BlomVanDiepenMooij_JMLR_21}, who proposed the refined notion $\doit(f_j : X_v = \xi_v)$ to resolve the ambiguity. This refinement is essential for systems like the bathtub, where the same target variable value can be achieved through different mechanisms with different causal consequences. By performing a \emph{complete} analysis of the causal semantics of the bathtub system under such interventions, we lend further credibility to their claim that this refined notion is a natural representation of ``elementary'' interventions.

A key contribution of the present paper is that the partially oriented bipartite graph $\vv{G}$ encodes \emph{both} the causal structure and the conditional independence structure in a single object: the cluster structure and edge directions encode the causal ordering, while $B$-separation encodes conditional independences. 
This is more convenient, as it avoids the need to construct and switch between multiple graphs, and it retains the equation nodes so that the intervention structure remains directly visible. 
Furthermore, our $B$-separation Markov property is more powerful than the Markov properties derived by \citet{BlomVanDiepenMooij_JMLR_21}:
by also exploiting the functional determinism among the variables, further conditional independences are obtained.\footnote{Our results also imply that these could alternatively be obtained by using $D$-separation in the Markov ordering graph.} Finally, our extended Markov property, which handles Markov kernels with non-random inputs through transitional conditional independence \citep{Forre2021}, has no counterpart in that work. We believe that these features may also facilitate future extensions and applications.

The BGCM framework is closely related to how engineers reason about causality \citep{Frisk++2012,BunusFritzson2002,KrysanderNyberg2002}. In modeling languages such as Modelica, systems are specified as sets of ``acausal'' equations, and causality is derived automatically through symbolic analysis---precisely the kind of analysis formalized by Simon's causal ordering algorithm and the BGCM framework.

When the bipartite graph does not admit a perfect matching, the Dulmage-Mendelsohn decomposition \citep{DulmageMendelsohn1958} provides a useful generalization that can represent overcomplete subsystems (more equations than variables) and incomplete subsystems (more variables than equations). \citet{BlomVanDiepenMooij_JMLR_21} demonstrate how marginal Markov properties for the complete and overcomplete subsystems can still be derived in this general setting.

To our knowledge, there is little other work using bipartite graphs for causal inference.
\citet{ZiglerPapadogeorgou2021} introduce a framework for estimating causal effects under interference when the treated units are distinct from the observed units.
\citet{Sharifian++2025} prove that every valid graph in the observational equivalence class of linear Gaussian cyclic SCMs corresponds to a perfect matching.


\section{Conclusion}\label{sec:conclusion}

We have proposed bipartite graphical causal models (BGCMs) as a causal modeling framework that uses bipartite graphs with equation nodes and variable nodes. This framework offers several advantages. First, it reduces the ambiguity inherent in specifying interventions by requiring that hard interventions explicitly reference the equation being replaced. Second, Simon's causal ordering algorithm provides a principled method for deriving the partial orientation of the bipartite graph, which encodes the causal structure. Third, the $B$-separation criterion and the resulting Markov property propagate conditional independences along the partial ordering. Fourth, the Markov property facilitates causal reasoning about domain invariances, providing a generalization of Pearl's do-calculus to BGCMs.

BGCMs naturally model equilibrium systems such as the bathtub, and can be applied to a wide range of other systems, including equilibrated economic markets (e.g., supply--demand systems, see Appendix~\ref{app:supply-demand}), electronic circuits, and biochemical reaction networks \citep{BlomMooij_JCI_23}. Directions for future work include dynamical extensions incorporating (stochastic) differential equations, and the development of structure learning algorithms for BGCMs.

\begin{acknowledgements} 
  I thank Claude Code (Opus 4.6--4.8) for assistance in the writing process, and Kathy Molenaar for useful discussions.
\end{acknowledgements}

\bibliography{uai2026,2026_cambridge,mooij}

@inproceedings{Spirtes95,
  title = {Directed Cyclic Graphical Representations of Feedback Models},
  author = {Peter Spirtes},
  booktitle = {Proceedings of the Eleventh Conference on {U}ncertainty in {A}rtificial {I}ntelligence ({UAI}-95)},
  pages = {499--506},
  year = 1995,
  month = 8, 
}

@article{Dawid1979,
  title={Conditional Independence in Statistical Theory},
  author={Dawid, A. Philip},
  journal={Journal of the Royal Statistical Society: Series {B} (Methodological)},
  volume={41},
  number={1},
  pages={1--15},
  year={1979},
}

@book{Nayak1995,
author = {Nayak, P.},
title = {Automated modeling of physical systems},
year = {1995},
publisher = {Springer-Verlag},
address = {Berlin}
}

@article{DulmageMendelsohn1958,
author = {Dulmage, A. L. and Mendelsohn, N. S.},
year = 1958,
title = {Coverings of bipartite graphs},
journal = {Canadian Journal of Mathematics},
volume = 10,
pages = {517-534},
}

@incollection{Simon1953,
author = {Simon, Herbert A.},
title = {Causal Ordering and Identifiability},
booktitle = {Studies in Econometric Methods},
pages = {49--74},
editors = {Hood, William C. and Koopmans, Tjalling C.},
publisher = {John Wiley \& Sons},
year = {1953}
}

@article{Forre2021,
author={Patrick Forr{\'e}},
journal = {arXiv.org preprint},
title={Transitional Conditional Independence},
volume={arXiv:2104.11547 [math.ST]},
year={2021},
month=Apr,
url={https://arxiv.org/abs/2104.11547}
}

@article{IwasakiSimon1994,
title={Causality and model abstraction},
author={Iwasaki, Yumi and Simon, Herbert A},
journal={Artificial intelligence},
volume={67},
issue={1},
pages={143--194},
year={1994},
publisher={Elsevier}
}

@Book{Pearl2009,
author = {Judea Pearl},
title = {Causality: Models, Reasoning and Inference},
year = {2009},
publisher = {Cambridge University Press},
address = {}
}

@article{Bongers++_AOS_21,
  author = {Stephan Bongers and Patrick Forr{\'e} and Jonas Peters and Joris M. Mooij},
  title = {Foundations of Structural Causal Models with Cycles and Latent Variables},
  journal = {Annals of Statistics},
  year = {2021},
  volume = {49},
  number = {5},
  pages = {2885-2915},
  doi = {10.1214/21-AOS2064},
}

@article{BlomVanDiepenMooij_JMLR_21,
  author  = {Tineke Blom and Mirthe M. van Diepen and Joris M. Mooij},
  title   = {Conditional Independences and Causal Relations implied by Sets of Equations},
  journal = {Journal of Machine Learning Research},
  year    = {2021},
  volume  = {22},
  number  = {178},
  pages   = {1-62},
  url     = {http://jmlr.org/papers/v22/20-863.html}
}

@inproceedings{BlomBongersMooij_UAI_19,
  title     = {Beyond Structural Causal Models: Causal Constraints Models},
  author    = {Tineke Blom and Stephan Bongers and Joris M. Mooij},
  pages     = {585-594},
  url       = {https://proceedings.mlr.press/v115/blom20a.html},
  booktitle = {Proceedings of the 35th Uncertainty in Artificial Intelligence Conference ({UAI}-19)},
  editor    = {Adams, Ryan P. and Gogate, Vibhav},
  year      = 2020,
  publisher = {PMLR},
  series    = {Proceedings of Machine Learning Research},
  month     = 7,
  volume    = 115,
}

@article{ForreMooij_1710.08775,
  author = "Patrick Forr{\'e} and Joris M. Mooij",
  journal = "arXiv.org preprint",
  title = "Markov Properties for Graphical Models with Cycles and Latent Variables",
  volume = "arXiv:1710.08775 [math.ST]",
  year = 2017,
  month = Oct,
  url = "https://arxiv.org/abs/1710.08775",
}

@misc{ForreMooij_Causality_2025,
  author       = {Patrick Forr\'e and Joris M. Mooij},
  title        = {A Mathematical Introduction to Causality},
  year         = 2025,
  note         = {Lecture Notes},
  url          = {https://staff.fnwi.uva.nl/j.m.mooij/articles/causality_lecture_notes_2025.pdf}
}

@article{BlomMooij_JCI_23,
  author = {Tineke Blom and Joris M. Mooij},
  journal = {Journal of Causal Inference},
  title = {Causality and Independence in Perfectly Adapted Dynamical Systems},
  volume = 11,
  issue = 1,
  pages = {20210005},
  year = 2023,
  url = {https://www.degruyter.com/document/doi/10.1515/jci-2021-0005/html},
}

@book{Kechris95,
    AUTHOR = {Kechris, Alexander S.},
     TITLE = {Classical Descriptive Set Theory},
    SERIES = {Graduate Texts in Mathematics},
    VOLUME = {156},
 PUBLISHER = {Springer-Verlag, New York},
      YEAR = {1995},
}

@article{GeigerVermaPearl1990,
  author    = {Dan Geiger and Thomas Verma and Judea Pearl},
  title     = {Identifying independence in {B}ayesian networks},
  journal   = {Networks},
  volume    = {20},
  number    = {5},
  pages     = {507--534},
  year      = {1990},
}

@inproceedings{VermaPearl1988,
  author    = {Thomas Verma and Judea Pearl},
  title     = {Causal networks: Semantics and expressiveness},
  booktitle = {Proceedings of the Fourth Conference on Uncertainty in Artificial Intelligence ({UAI})},
  pages     = {352--359},
  year      = {1988},
}

@article{LauritzenDawidLarsenLeimer1990,
  author    = {S.L. Lauritzen and A.P. Dawid and B.N. Larsen and H.-G. Leimer},
  title     = {Independence properties of directed {M}arkov fields},
  journal   = {Networks},
  volume    = {20},
  number    = {5},
  pages     = {491--505},
  year      = {1990},
}

@book{Lauritzen1996,
  title = {Graphical Models},
  author = {S.L. Lauritzen},
  publisher = {Clarendon Press},
  address = {Oxford},
  year = 1996,
  series = {Oxford Statistical Science Series},
  volume = 17,
}

@article{ZiglerPapadogeorgou2021,
  title = {Bipartite Causal Inference with Interference},
  author = {Corwin M. Zigler and Georgia Papadogeorgou},
  journal = {Statistical Science},
  year = 2021,
  volume = 36,
  number = 1,
  pages = {109-123},
  doi = {10.1214/19-STS749},
}

@inproceedings{Sharifian++2025,
 author = {Sharifian, Ehsan and Salehkaleybar, Saber and Kiyavash, Negar},
 booktitle = {Advances in Neural Information Processing Systems ({NeuRIPS 2025})},
 editor = {D. Belgrave and C. Zhang and H. Lin and R. Pascanu and P. Koniusz and M. Ghassemi and N. Chen},
 pages = {63520--63538},
 publisher = {Curran Associates, Inc.},
 title = {Near-Optimal Experiment Design in Linear non-{G}aussian Cyclic Models},
 volume = {38},
 year = {2025},
}

@inproceedings{BunusFritzson2002,
  author = {P. Bunus and P. Fritzson},
  booktitle = {2nd International Modelica Conference Proceedings},
  year = 2002,
  pages = {157-165},
  title = {Methods for Structural Analysis and Debugging of Modelica Models},
}

@article{Frisk++2012,
  author={Frisk, Erik and Bregon, Anibal and Aslund, Jan and Krysander, Mattias and Pulido, Belarmino and Biswas, Gautam},
  journal={IEEE Transactions on Systems, Man, and Cybernetics - Part A: Systems and Humans},
  title={Diagnosability Analysis Considering Causal Interpretations for Differential Constraints},
  year={2012},
  volume={42},
  number={5},
  pages={1216-1229},
  doi={10.1109/TSMCA.2012.2189877},
}

@article{KrysanderNyberg2002,
title = {Structural Analysis for Fault Diagnosis of DAE Systems Utilizing MSS Sets},
journal = {IFAC Proceedings Volumes},
volume = {35},
number = {1},
pages = {143-148},
year = {2002},
note = {15th IFAC World Congress},
issn = {1474-6670},
doi = {10.3182/20020721-6-ES-1901.00755},
url = {https://www.sciencedirect.com/science/article/pii/S147466701539176X},
author = {Mattias Krysander and Mattias Nyberg},
}

@article{HopcroftKarp1973,
title = {An $n^{5/2}$ algorithm for maximum matchings in bipartite graphs},
author = {Hopcroft, John E. and Karp, Richard M.},
year = {1973},
journal = {SIAM Journal on Computing},
volume = 2,
issue = 4,
pages = {225-231},
doi = {10.1137/0202019},
}

\newpage

\onecolumn

\title{Causal Reasoning with Bipartite Graphical Causal Models\\(Supplementary Material)}
\maketitle

This Supplementary Material contains proofs of the main results and additional details.

\appendix

\section{Measurability of Solution Functions}\label{app:measurable-solutions}

We use the following standard measurable-graph fact to justify that the uniquely
defined solution functions appearing in the main text are measurable.

\begin{lemma}[Measurability of uniquely defined solution maps]\label{lem:measurable_solution}
  Let $\mathcal{X}$ and $\mathcal{Y}$ be standard Borel spaces, let
  $\mathcal{Z}$ be a measurable space, let $z_0 \in \mathcal{Z}$ be such that
  $\{z_0\}$ is measurable, and let
  $H:\mathcal{X}\times\mathcal{Y}\to\mathcal{Z}$ be measurable. Suppose that
  for every $x\in\mathcal{X}$ there exists a unique $y=:\psi(x)\in\mathcal{Y}$
  such that $H(x,y)=z_0$. Then $\psi:\mathcal{X}\to\mathcal{Y}$ is measurable.
\end{lemma}
\begin{proof}
  The graph of $\psi$ is
  \[
    \Gamma_\psi
    = \{(x,y)\in\mathcal{X}\times\mathcal{Y}: H(x,y)=z_0\}.
  \]
  This set is measurable because it is the inverse image of $\{z_0\}$ under the
  measurable map $(x,y)\mapsto H(x,y)$.
  Hence, by \citep[14.12]{Kechris95}, $\psi$ is measurable.
\end{proof}

\begin{corollary}[Measurability of cluster solution functions]\label{cor:cluster_solution_measurable}
  Under Assumption~\ref{ass:leu}, suppose an endogenous cluster $[c]$ is uniquely
  solvable in the sense of Definition~\ref{def:cluster_uniquely_solvable}. Then
  its local solution function
  \[
    \Phi^{[c]}:\mathcal{X}_{\pa_{\vv{G}}([c])}\to\mathcal{X}_{[c]\cap V}
  \]
  is measurable.
\end{corollary}
\begin{proof}
  Write
  $\mathcal{X}:=\mathcal{X}_{\pa_{\vv{G}}([c])}$ and
  $\mathcal{Y}:=\mathcal{X}_{[c]\cap V}$. For each equation
  $f_j\in F\cap[c]$, the variables occurring in $f_j$ are contained in
  $([c]\cap V)\cup\pa_{\vv{G}}([c])$. Hence the measurable equation map
  $\phi_j$ induces a measurable function of $\mathcal{X}\times \mathcal{Y} \to \mathbb{R}$. 
  Collect these equations into the measurable map
  \[
    H : \mathcal{X} \times \mathcal{Y} \to \mathbb{R}^{F\cap[c]} : (x,y) \mapsto (\phi_j(x,y))_{j\in F\cap[c]}.
  \]
  Unique solvability says that for every
  parent value $x\in\mathcal{X}$ there is a unique $y=\Phi^{[c]}(x)$ such that
  $H(x,y)=0$. Lemma~\ref{lem:measurable_solution} therefore implies that
  $\Phi^{[c]}$ is measurable.
\end{proof}
Since the graph has finitely many clusters, recursive substitution of the measurable
local solution functions along the causal ordering shows that the global solution
function $\Phi:\mathcal{X}_U\to\mathcal{X}_V$ of Proposition~\ref{prop:global-solution}
is measurable as well.

\section{Preliminaries on $d$-Separation and $D$-Separation}\label{app:preliminaries}

We recall the standard notion of $d$-separation, introduced by \citet{VermaPearl1988} (see also \citealp{Lauritzen1996,Pearl2009}), and, for graphs with deterministic relations, that of $D$-separation \citep{GeigerVermaPearl1990}.

\begin{definition}[$d$-blocking]\label{def:d-blocking}
  A walk $v_1 \dots v_k$ on a DAG $G$ is $d$-blocked by $C \subseteq V$, if it contains:
  \begin{itemize}
    \item a \emph{collider} $v_{i-1} \to v_i \ot v_{i+1}$ with $v_i \notin \anc_G(C)$, or
    \item a \emph{non-collider} (possibly endpoint) $v_i \in C$.
  \end{itemize}
  Otherwise, the walk is called \emph{$d$-open given $C$}.
\end{definition}
\begin{definition}\label{def:d-sep}
Let $A, B, C \subseteq V$ be sets of variable nodes in a DAG $G$. We say that \emph{$A$ and $B$ are $d$-separated given $C$ in $G$}, written
\[
A \Perp^{d}_{G} B \given C,
\]
if every walk from a node in $A$ to a node in $B$ is $d$-blocked by $C$ in $G$.\footnote{It suffices if every path from a node in $A$ to a node in $B$ is $d$-blocked by $C$ in $G$, yielding an equivalent formulation of $d$-separation that is easier to check manually.}
\end{definition}

\begin{definition}[Functionally determined in an acyclic SCM]\label{def:functionally_determined_DAG}
  Let $C \subseteq V$ be a subset of nodes in an acyclic SCM with graph $G$, with exogenous nodes $U$ and endogenous nodes $V \sm U$.
  Define $C_0 := C$ and
  \[ C_{n+1} := C_n \cup \{ v \in V \sm U : \pa_{G}(v) \subseteq C_n \}. \]
  We define $\fdet_{G}(C) := \bigcup_{n \ge 0} C_n$ and refer to those as the nodes that are \emph{functionally determined by $C$}.
\end{definition}
The following definition is inspired by \citep{GeigerVermaPearl1990}.
\begin{definition}[$D$-separation]\label{def:D-sep}
Let $A,B,C\subseteq V$ be sets of nodes in an acyclic SCM with graph $G$ (exogenous nodes $U$, endogenous $V\sm U$). We say that \emph{$A$ and $B$ are $D$-separated given $C$ in $G$}, written $A \Perp^{D}_{G} B \given C$, if $A \Perp^{d}_{G} B \given \fdet_{G}(C)$ (formulation~$(2)$ of Lemma~\ref{lem:D-sep}).
\end{definition}
This formulation of $D$-separation (which corresponds with formulation (2) in the following lemma) is equivalent to two other formulations:
\begin{lemma}\label{lem:D-sep}
  Let $A,B,C$ be sets of nodes in an acyclic SCM with graph $G$, with exogenous nodes $U$ and endogenous nodes $V \sm U$.
  Let $\fdet_{G}(C)$ be the nodes in the DAG that are functionally determined by $C$ (as in Definition~\ref{def:functionally_determined_DAG}).
  The following three formulations of $D$-separation are equivalent:
  \begin{enumerate}
    \item all walks between a node in $A$ and a node in $B$ contain
      \begin{enumerate}
        \item a collider not in $\anc_G(C)$
        \item a non-endpoint non-collider in $\fdet_{G}(C)$
        \item an end node in $\fdet_{G}(C)$
      \end{enumerate}
    \item all walks between a node in $A$ and a node in $B$ contain
      \begin{enumerate}
        \item a collider not in $\anc_G(\fdet_{G}(C))$
        \item a non-endpoint non-collider in $\fdet_{G}(C)$
        \item an end node in $\fdet_{G}(C)$
      \end{enumerate}
      ($\iff$ $A$, $B$ are $d$-separated given $\fdet_{G}(C)$)
    \item all walks between a node in $A$ and a node in $B$ contain
      \begin{enumerate}
        \item a collider not in $\anc_G(C)$
        \item a non-endpoint non-collider in $C$
        \item an end node in $\fdet_{G}(C)$
        \item a fork in $\fdet_{G}(C)$
      \end{enumerate}
  \end{enumerate}
\end{lemma}
\begin{proof}
Write $\bar C:=\fdet_G(C)$. From Definition~\ref{def:functionally_determined_DAG} we have $C\subseteq\bar C$ and
\begin{equation}\label{eq:fdet-parents}
  n\in\bar C\setminus C \;\Longrightarrow\; n\in V\setminus U \text{ and } \pa_G(n)\subseteq\bar C ;
\end{equation}
in particular $\anc_G(C)\subseteq\anc_G(\bar C)$. On a walk $\pi$ we call an internal node a \emph{collider} if both incident edges point into it, a \emph{fork} if both point out of it, and a \emph{chain} if one points in and one out; the two end nodes are treated separately. A \emph{parent-neighbor} of a node $n$ on $\pi$ is a neighbor $p$ on $\pi$ with $p\to n$ in $G$; thus a collider has two parent-neighbors, a chain has one, and a fork has none. Call $\pi$ \emph{$j$-active} if it is not blocked according to formulation $(j)$. We prove that the three notions of ``active'' coincide on every walk $\pi$ between $A$ and $B$; the equivalence of the three ``all walks are blocked'' statements is then immediate.

\emph{$(1)\Leftrightarrow(2)$.} The two criteria differ only in the collider clause, and $\anc_G(C)\subseteq\anc_G(\bar C)$, so a collider outside $\anc_G(\bar C)$ is also outside $\anc_G(C)$; as the other clauses coincide, every walk blocked under $(2)$ is blocked under $(1)$, i.e., every $1$-active walk is $2$-active.

Conversely, let $\pi$ be $2$-active: every collider lies in $\anc_G(\bar C)$, and no node of $\bar C$ occurs on $\pi$ as a non-collider or as an end node. Let $k$ be a collider of $\pi$; we show $k\in\anc_G(C)$. If $k\in\bar C$ then in fact $k\in C$: otherwise \eqref{eq:fdet-parents} gives $\pa_G(k)\subseteq\bar C$, so the two parent-neighbors of $k$ would be nodes of $\bar C$ occurring as non-colliders or end nodes, contradicting $2$-activity; hence $k\in C\subseteq\anc_G(C)$. If $k\notin\bar C$, pick a shortest directed path $k\to n_1\to\dots\to n_t$ with $n_t\in\bar C$ (one exists since $k\in\anc_G(\bar C)$), so that $n_1,\dots,n_{t-1}\notin\bar C$. Were $n_t\in\bar C\setminus C$, then \eqref{eq:fdet-parents} would place its predecessor on the path---$n_{t-1}$, or $k$ if $t=1$---in $\pa_G(n_t)\subseteq\bar C$, contradicting the choice of that predecessor outside $\bar C$. Hence $n_t\in C$ and $k\in\anc_G(C)$. So every collider of $\pi$ lies in $\anc_G(C)$ and $\pi$ is $1$-active.

\emph{$(1)\Leftrightarrow(3)$.} The collider clause and the end-node clause are identical in the two formulations. If $\pi$ is $1$-active it has no internal non-collider in $\bar C$; in particular it has no internal non-collider in $C$ and no fork in $\bar C$, so $\pi$ is $3$-active.

Conversely, let $\pi$ be $3$-active. To show it is $1$-active it suffices to rule out internal non-colliders in $\bar C$. Forks in $\bar C$ are excluded by $3$-activity, and chains in $C$ are excluded as well, so the only remaining possibility is a chain node $m\in\bar C\setminus C$; suppose one occurs. Define $q_0:=m$ and let $q_{i+1}$ be the parent-neighbor of $q_i$, continuing as long as $q_i$ is a chain in $\bar C\setminus C$ (which by \eqref{eq:fdet-parents} guarantees a parent-neighbor in $\bar C$). Each $q_i$ lies in $\bar C$, and $q_{i+1}\to q_i\to\dots\to q_0$ is a directed path in the DAG $G$, so the $q_i$ are distinct and the process stops, at some $q_k\in\bar C$. Since $q_k$ has an edge pointing out of it (toward $q_{k-1}$), it is not a collider. If $q_k$ is an end node, then an end node lies in $\bar C$; if $q_k$ is a fork, then a fork lies in $\bar C$; and if $q_k$ is a chain, then---the process having stopped---$q_k\in C$, so an internal non-collider lies in $C$. Each case contradicts $3$-activity. Hence no chain node of $\pi$ lies in $\bar C\setminus C$, so $\pi$ has no internal non-collider in $\bar C$ and is $1$-active.

Thus the three notions of ``active'' (equivalently, of ``blocked'') coincide on every walk, and the three formulations of $D$-separation are equivalent. 
By definition, formulation $(2)$ is $d$-separation of $A$ and $B$ given $\bar C$.
\end{proof}
\citet{GeigerVermaPearl1990} defined $D$-separation (restricted to disjoint $A,B,C$) with formulation $(3)$, and showed that it is equivalent to formulation $(1)$.
What Lemma~\ref{lem:D-sep} adds is the equivalence with formulation~$(2)$: $D$-separation given $C$ coincides with ordinary $d$-separation given the enlarged conditioning set $\bar C=\fdet_G(C)$.\footnote{This equivalence was observed by Claude Code.} This reduction is useful because it lets us fall back on the theory of $d$-separation, which is considerably wider in scope than that of $D$-separation: it extends to cyclic systems through $\sigma$-separation \citep{ForreMooij_1710.08775} and underpins a broad range of Markov-property, completeness, and algorithmic results that have no direct $D$-separation counterpart. 

\section{Proof of the Markov Property}\label{app:direct-proof}

Our strategy to prove the Markov property for BGCMs (Theorem~\ref{thm:Markov_B}) will be as follows.

We will first ignore deterministic relations and prove a weaker Markov property using a separation notion that we call $b$-separation (lowercase $b$ for ``bipartite'').
This separation notion is designed to correspond to $d$-separation on the \emph{acyclification}, a directed acyclic graph constructed from the partially ordered bipartite graph.
This mimics the acyclification strategy for cyclic SCMs \citep{Spirtes95,ForreMooij_1710.08775,Bongers++_AOS_21}.
However, we do not make the clusters (corresponding to strongly connected components in SCMs) fully connected, because we typically work with the augmented graph that contains all nodes, including exogenous random variable nodes, and there is no reason to assume a latent noise source feeds into a cycle.
From the equivalence of $b$-separation on the partially oriented bipartite graph and $d$-separation on its acyclification (Lemma~\ref{lem:b-sep_equiv_d-sep}), we then prove a $b$-separation Markov property 
(Theorem~\ref{thm:Markov_b}) by reduction to the standard Markov property for acyclic SCMs.

We then strengthen the Markov property by taking into account determinism, analogous to how $D$-separation \citep{GeigerVermaPearl1990} strengthens $d$-separation in Bayesian networks.
The key intuition is: once every parent of a cluster is fixed by the conditioning information, all variables in the cluster are fixed too.
This leads directly to the notion of $B$-separation.
Similarly to how $D$-separation is related to $d$-separation, $B$-separation can be expressed in terms of $b$-separation (Lemma~\ref{lem:B-eq-b-fdet}).
This observation allows us to obtain Theorem~\ref{thm:Markov_B} as a Corollary of Theorem~\ref{thm:Markov_b}.

\subsection{$b$-Separation Markov Property}

We repeatedly make use of the following elementary consequences of the definitions:
\begin{itemize}
  \item Directed edges in $\vv{G}$ always point from variable to equation;
  \item Parents of a cluster are variables;
  \item For a walk between two variable nodes, all segment exits will be variable nodes;
  \item The end nodes of the walk are always exits of their segments.
\end{itemize}
We will also use that clusters are connected by double-undirected edges.
\begin{lemma}[Double-edge connectivity of clusters]\label{lem:cluster_double_connected}
  Let $a,b \in V \cup F$. If $a \sim b$, then there exists a possibly trivial walk from
  $a$ to $b$ in $\vv{G}$ that uses only $\otto$ edges and whose nodes all lie in $[a]=[b]$.
\end{lemma}
\begin{proof}
  It suffices to prove the claim for each generating relation in
  Definition~\ref{def:equivalence}, since walks can then be concatenated along a finite
  chain of such relations. If $a \ttt b$ is a matched edge, then $a \sim b$, so this edge is
  oriented as $a \otto b$ in $\vv{G}$ by Definition~\ref{def:partial-orientation}. If $a$
  and $b$ lie on a common closed $M$-alternating walk, take the subwalk of that closed walk
  from $a$ to $b$. Every edge on this subwalk has both endpoints on the same closed
  $M$-alternating walk, hence its endpoints are equivalent; by
  Definition~\ref{def:partial-orientation}, each such edge is therefore oriented as
  $\otto$. The resulting walk stays inside the common equivalence class.
\end{proof}
We first define the appropriate acyclification.
\begin{definition}
  For a partially oriented bipartite graph $\vv{G}$ with variable nodes $V$ and equation nodes $F$, we define its \emph{acyclification} as the directed acyclic graph $\vv{G}^\acy = (V,E^\acy)$ with nodes $V$ and with a directed edge $v \to v'$ for $v,v' \in V$ if and only if $v \in \pa_{\vv{G}}([v'])$.
\end{definition}
First we show that ``anterior in $\vv{G}$'' (at the node level) corresponds to ``ancestral in $\vv{G}^\acy$'' (at the cluster level).
\begin{lemma}\label{lem:acyclification_ancestral}
  For nodes $a,b \in V$:
  \[ a \in \ant_{\vv{G}}(b) \iff [a]\cap\anc_{\vv{G}^\acy}(b)\neq\emptyset. \]
\end{lemma}
\begin{proof}
  Let $\pi$ be an anterior walk in $\vv{G}$, that is, a walk of the form
  \[ s_{1,1} \otto \dots \otto s_{1,k_1} \to s_{2,1} \otto \dots \otto s_{2,k_2} \to \dots \to s_{m,1} \otto \dots \otto s_{m,k_m} \]
  which we partitioned into maximal subwalks $s_i$ of equivalent nodes, each $s_i$ being of the form $s_{i,1} \otto \dots \otto s_{i,k_i}$ (with possibly $k_i = 1$).
  We project it onto a directed walk in $\vv{G}^\acy$ by picking from each segment $s_i$ the outgoing node $s_{i,k_i}$:
  \[ s_{1,k_1} \to s_{2,k_2} \to \dots \to s_{m,k_m}. \]
  Hence, if $s_{1,1}$ is anterior to $s_{m,k_m}$ in $\vv{G}$, then $s_{1,k_1}$ is an ancestor of $s_{m,k_m}$ in $\vv{G}^\acy$.
  Since $s_{1,k_1} \in [s_{1,1}]$, the claim follows.

  Vice versa, let $\pi^\acy$ be a directed walk in $\vv{G}^\acy$:
  \[ v_1 \to \dots \to v_m. \]
  By definition, each edge in $\pi^\acy$ connects variables in different clusters of $\vv{G}$.
  The edge $v_i \to v_{i+1}$ in $\vv{G}^\acy$ (with $v_i \in \pa_{\vv{G}}([v_{i+1}])$) can be lifted to a walk in $\vv{G}$ as follows.
  Choose an equation $f_i \in F \cap [v_{i+1}]$ with $v_i \to f_i$ in $\vv{G}$, and connect $f_i$ to $v_{i+1}$ by a $\otto$-walk within $[v_{i+1}]$ using Lemma~\ref{lem:cluster_double_connected}, resulting in the lift $v_i \to f_i \otto \cdots \otto v_{i+1}$ in $\vv{G}$.
  Concatenating these lifts yields an anterior walk in $\vv{G}$ from $v_1$ to $v_m$.
  By concatenating this with the double-edge walk from Lemma~\ref{lem:cluster_double_connected}, we may obtain an anterior walk in $\vv{G}$ from any node in $[v_1]$ to $v_m$.
  Hence, if a node in $[a]$ is ancestral to $b$ in $\vv{G}^\acy$, then $a$ itself is anterior to $b$ in $\vv{G}$.
\end{proof}

The following notion is related to the segment-based version of $\sigma$-separation \citep{ForreMooij_1710.08775}, but strengthens it by adding another way in which non-collider segments can block (the third rule).
\begin{definition}[$b$-blocking]\label{def:b-blocking}
  For $C \subseteq V$, the walk is called \emph{$b$-blocked by $C$} if it contains:
  \begin{enumerate}
    \item a collider segment that does not intersect $\ant_{\vv{G}}(C)$, or
    \item a non-collider segment that has an exit in $C$, or
    \item a non-collider segment with two distinct exits whose cluster has all its $\vv{G}$-parents in $C$.
  \end{enumerate}
  Otherwise, the walk is called \emph{$b$-open given $C$}.
\end{definition}
Note: Since the two end nodes qualify as exits, an end node in $C$ always $b$-blocks the walk.
\begin{definition}[$b$-separation]\label{def:b-sep}
Let $A, B, C \subseteq V$ be sets of variable nodes. We say that \emph{$A$ and $B$ are $b$-separated given $C$ in $\vv{G}$}, written
\[
A \Perp^{b}_{\vv{G}} B \given C,
\]
if every walk from a node in $A$ to a node in $B$ is $b$-blocked by $C$ in $\vv{G}$.\footnote{Thanks to rule~3 of Definition~\ref{def:b-blocking}, it is equivalent to require only that every \emph{path} from $A$ to $B$ be $b$-blocked by $C$ (Lemma~\ref{lem:b-sep_walk_path}); the path formulation is usually more convenient to check by hand.}
\end{definition}

The three rules in which $b$-separation blocks a walk mirror, segment by segment, the way $d$-separation blocks the corresponding structure in the acyclification $\vv{G}^\acy$ (Lemma~\ref{lem:b-sep_equiv_d-sep}): a collider segment projects to a collider whose center is a common child; a non-collider segment with a single exit $v$ (a chain, an endpoint, or a one-node fork $\ot v \to$) projects to a chain/fork centered at $v$, blocked iff $v \in C$; and a non-collider segment with two \emph{distinct} exits $a \neq b$ projects to a fork $a \ot p \to b$ through a common parent $p$, blocked iff $a \in C$ or $b \in C$ or every such $p$ lies in $C$.\footnote{Note that a one-node fork $\ot v \to$ has a single (distinct) exit and is thus governed by rule~2 only, not rule~3.}

We added the third rule to make $b$-separation via walks coincide with $b$-separation via paths.
\begin{lemma}[$b$-separation via walks or paths]\label{lem:b-sep_walk_path}
For all $A, B, C \subseteq V$, every walk from $A$ to $B$ is $b$-blocked by $C$ if and only if every path from $A$ to $B$ is $b$-blocked by $C$.
\end{lemma}
\begin{proof}
Since paths are walks, ``all walks $b$-blocked'' implies ``all paths $b$-blocked''.

  Conversely, suppose there exists a $b$-open walk from $A$ to $B$, and among all such walks
  with the same end nodes choose one, say $\pi$, of minimal length. We show that $\pi$ has no
  repeated node, and hence is a $b$-open path.

  First note that no segment of $\pi$ contains the same node twice. Indeed, if a segment
  contains two occurrences of a node $y$, deleting the closed $\otto$-subwalk between these
  two occurrences gives a strictly shorter walk with the same end nodes. Only this segment is
  changed; its bounding directed edges, exits, and cluster remain the same. Hence rules~2
  and~3 of Definition~\ref{def:b-blocking} have the same truth value before and after the
  deletion. If the segment is a collider, then, since $\pi$ is $b$-open, it meets
  $\ant_{\vv{G}}(C)$. But all nodes in a segment are connected by $\otto$-walks, so if one
  node of the segment is anterior to $C$, then every node of the segment is anterior to $C$.
  Thus the shortened collider segment is still activated. The shortened walk is therefore
  $b$-open, contradicting the minimality of $\pi$.

  Now suppose, for contradiction, that $\pi$ nevertheless visits some node twice. 
  
  Write $\pi = n_0, e_1, \dots, e_k, n_k$, let $x = n_\mu = n_\nu$ with $\mu < \nu$ and let $\pi'$ be obtained by deleting $e_{\mu+1}, \dots, n_\nu$, i.e., $\pi' = n_0, \dots, n_\mu, e_{\nu+1}, n_{\nu+1}, \dots, n_k$. 
  This is a valid walk, since $e_{\nu+1}$ joins $n_\nu = n_\mu$ to $n_{\nu+1}$, and its end nodes $n_0, n_k$ are unchanged.
  The two occurrences of $x$ lie in distinct segments of $\pi$, since no segment of $\pi$ contains a repeated node.

Let $s'$ be the segment of $\pi$ containing the occurrence $n_\mu$, and $s''$ the segment
containing $n_\nu$; both lie in $[x]$. In $\pi'$ the part of $s'$ from its left boundary to
$n_\mu$ and the part of $s''$ from $n_\nu$ to its right boundary merge into a single segment
$t \subseteq [x]$ whose left bounding edge is that of $s'$ (if applicable) and whose right bounding edge is
that of $s''$ (if applicable). Hence $t$ has a left exit iff $s'$ does, and a right exit iff $s''$ does.
Every other segment of $\pi'$ coincides with a segment of $\pi$ (same nodes, bounding edges,
exits, and type), so $\pi'$ can fail to be $b$-open only at $t$. We show $t$ does not
$b$-block. 

\emph{Rule 1 (collider).} Suppose $t$ is a collider, i.e., $s'$ has no left exit and $s''$
has no right exit. If $s'$ is itself a collider then, being a segment of the $b$-open $\pi$,
it meets $\ant_{\vv{G}}(C)$; as $s' \subseteq [x] = [t]$, so does $t$. The same holds if $s''$
is a collider. Otherwise $s'$ has a right exit $\rho'$ and $s''$ a left exit $\lambda''$, both
variables of $[x]$ whose exit edges point into strict descendant clusters. Hence the deleted
sub-walk leaves $[x]$ downward at $\rho'$ and re-enters it from below at $\lambda''$, so along
it some descending edge (one traversed from its variable into a child cluster) is immediately
followed---across a single segment---by an ascending edge; the first such segment $s^\dagger$
is a collider, and every directed edge before it descends, so $[x]$ is anterior to
$s^\dagger$. As $\pi$ is $b$-open, $s^\dagger$ meets $\ant_{\vv{G}}(C)$; since $[x]$ is
anterior to $s^\dagger$, so does $[x]$. 

\emph{Rule 2.} Every exit of $t$ is a left exit of $s'$ or a right exit of $s''$, hence an
exit of a segment of the $b$-open $\pi$, hence not in $C$.

\emph{Rule 3.} Suppose $t$ has two distinct exits and $\pa_{\vv{G}}([x]) \subseteq C$; we
derive a contradiction. Then $s'$ has a left exit $\ell$ and $s''$ a right exit $r$ with
$\ell \neq r$. Consider the right bounding edge of $s'$.
\begin{itemize}
  \item If it points out of $s'$, its endpoint $\rho'$ is a right exit of $s'$. 
    If $\rho' \neq \ell$, then $s'$ has two distinct exits and $\pa_{\vv{G}}([x]) \subseteq C$, so $s'$ already $b$-blocks $\pi$---a contradiction. 
    If $\rho' = \ell$, then $s' = \{\ell\}$ is a single node, because no segment of $\pi$ contains a repeated node; hence $\ell=x$. 
    Now $x$ also lies on $s''$, while $s''$ has right exit $r\neq x$. 
    If the left bounding edge of $s''$ pointed out of $s''$, its left exit would be distinct from $r$ (otherwise $s''$ would repeat $r$, or would be the single node $r$, both impossible since it also contains $x\neq r$), so $s''$ would have two distinct exits and would already $b$-block $\pi$. 
    Thus the left boundary of $s''$ is an equation entered by an edge $q \to \cdot$ with $q \in \pa_{\vv{G}}([x]) \subseteq C$; then $q$ is a right exit lying in $C$ of the segment preceding $s''$, which therefore $b$-blocks $\pi$---a contradiction.
  \item If it points into $s'$, the right boundary of $s'$ is an equation entered by an edge
    $\cdot \ot q'$ with $q' \in \pa_{\vv{G}}([x]) \subseteq C$; then $q'$ is a left exit lying
    in $C$ of the segment following $s'$, which $b$-blocks $\pi$---a contradiction.
\end{itemize}
Hence $t$ does not $b$-block, so $\pi'$ is $b$-open. 
This contradicts the minimality of $\pi$. 
Therefore the minimal $b$-open walk $\pi$ has no repeated node, i.e., it is a $b$-open path from $A$ to $B$.
\end{proof}

The following lemma shows that we have correctly designed $b$-separation such that it is equivalent to $d$-separation in the acyclification.
\begin{lemma}\label{lem:b-sep_equiv_d-sep}
For all $A, B, C \subseteq V$:
\[ A \Perp^{b}_{\vv{G}} B \given C \iff A \Perp^{d}_{\vv{G}^\acy} B \given C.  \]
\end{lemma}
\begin{proof}
``$\Rightarrow$'': $b$-separation implies $d$-separation in the acyclification.
By contrapositive: given a path $\pi^\acy = v_0, v_1, \ldots, v_k$ in $\vv{G}^\acy$ from $v_0 \in A$ to $v_k \in B$ that is $d$-open given $C$, we construct a walk $\pi$ in $\vv{G}$ between $v_0$ and $v_k$ that is $b$-open given $C$.
By definition, each edge in $\pi^\acy$ connects variables in different clusters of $\vv{G}$.
It can be lifted to a walk in $\vv{G}$ by lifting each directed edge in the same way as in the proof of Lemma~\ref{lem:acyclification_ancestral}:
  \begin{center}\begin{tabular}{ccc}
    $v_i \to v_{i+1}$ on $\pi^\acy$ & is lifted to & $v_i \to f_i \otto \cdots \otto v_{i+1}$ in $\vv{G}$ \\
    $v_i \ot v_{i+1}$ on $\pi^\acy$ & is lifted to & $v_i \otto \cdots \otto f_i \ot v_{i+1}$ in $\vv{G}$.
  \end{tabular}\end{center}
Concatenating these lifts yields a walk $\pi$ in $\vv{G}$ between $v_0$ and $v_k$.

We show that $\pi$ is $b$-open given $C$ by checking each possibility in which it could be blocked (cf.\ Definition~\ref{def:b-blocking}).
\begin{itemize}
  \item A \emph{collider segment} stems from the concatenated lifts $v_{i-1} \to f_{i-1} \otto \cdots \otto v_i \otto \cdots \otto f_i \ot v_{i+1}$ of some collider $v_{i-1} \to v_i \ot v_{i+1}$ on $\pi^\acy$.
    Since $\pi^\acy$ is $d$-open given $C$, $v_i \in \anc_{\vv{G}^\acy}(C)$, hence $v_i \in \ant_{\vv{G}}(C)$ (by Lemma~\ref{lem:acyclification_ancestral}).
    All nodes in the segment are anterior to $v_i$, and by transitivity, each node in the segment lies in $\ant_{\vv{G}}(C)$. 
    Thus the segment does not $b$-block.
  \item A \emph{non-collider segment} stems from a non-collider $v_i$ on $\pi^\acy$ (a chain, a fork, or an end node). The lift preserves outgoing edges on variable nodes, and it preserves end points.
    So $v_i$ must be an exit of the segment. The lifting procedure cannot generate a segment with two distinct exits.
    Since $\pi^\acy$ is $d$-open given $C$ and $v_i$ is a non-collider on $\pi^\acy$, we have $v_i \notin C$.
    Thus the segment does not $b$-block.
\end{itemize}
Hence, $\pi$ is $b$-open given $C$.

``$\Leftarrow$'': $d$-separation in the acyclification implies $b$-separation.
By contrapositive: given a walk $\pi = r_0, r_1, \ldots, r_n$ in $\vv{G}$ from $r_0 \in A$ to $r_n \in B$ that is $b$-open given $C$, we construct a walk $\pi^\acy$ in $\vv{G}^\acy$ between $r_0$ and $r_n$ that is $d$-open given $C$.
The walk $\pi$ consists of segments $s_1, \dots, s_m$ separated by directed edges; each such directed edge is of the form $v \to f$ or $f \ot v$ with $v \in V$, $f \in F$, and $v \in \pa_{\vv{G}}([f])$.
We project each segment $s_i$ to a piece (node $v_i$ or a walk $a_i \ot p_i \to b_i$), according to its type:
\begin{itemize}
  \item A \emph{collider segment} $s_i$ is $b$-open, so it intersects $\ant_{\vv{G}}(C)$; we pick a node $v_i \in [s_i] \cap V \cap \anc_{\vv{G}^\acy}(C)$ (which exists by Lemma~\ref{lem:acyclification_ancestral}) and project $s_i$ to $v_i$.
    On $\pi^\acy$, $v_i$ will become a collider $\to v_i \ot$, which is $d$-open given $C$ as $v_i \in \anc_{\vv{G}^\acy}(C)$.
  \item A \emph{non-collider segment with a single exit} $v_i$ (a chain, a one-node fork, or an end node $r_0$ resp.\ $r_n$): since $\pi$ is $b$-open, $v_i \notin C$; we project $s_i$ to $v_i$. 
    As $v_i$ either carries an outgoing edge leaving its cluster or is an end node, it is a non-collider (chain, fork, or endpoint) on $\pi^\acy$, and $d$-open since $v_i \notin C$.
  \item A \emph{non-collider segment with two distinct exits} $a_i \neq b_i$: since $\pi$ is $b$-open, $a_i, b_i \notin C$ and there must be a parent $p_i \in \pa_{\vv{G}}([s_{i,1}]) \sm C$; we project $s_i$ to the fork $a_i \ot p_i \to b_i$.
    As both $a_i, b_i$ either carry an outgoing edge leaving their cluster or form an end node, they are non-colliders on $\pi^\acy$, and both are $d$-open given $C$ because $a_i, b_i \notin C$.
    Additionally, $p_i$ does not $d$-block given $C$ on $\pi^\acy$ as $p_i \notin C$.
\end{itemize}
The consecutive pieces are joined together to form a walk, preserving the directed edges that separated the segments on $\pi$.
By construction, two consecutive variable nodes on this sequence must lie in different clusters, and one must be parent of the other in $\vv{G}^\acy$.
The resulting walk $\pi^\acy$ is $d$-open given $C$ by construction.
\end{proof}

We can now prove the global Markov property for partially oriented bipartite graphs via reduction to the Markov property for acyclic SCMs.
\begin{theorem}\label{thm:Markov_b}
  Suppose Assumption~\ref{ass:leu} holds.
  When assigning independent distributions to all exogenous variables, the resulting joint distribution $\Prb(X_V)$ satisfies: for all $A, B, C \subseteq V$, 
    \[ A \Perp^{b}_{\vv{G}} B \given C \implies X_A \Indep_{\Prb(X_V)} X_B \given X_C. \]
\end{theorem}
\begin{proof}
  The clusters of $\vv{G}$ are partially ordered by the directed edges between them. 
  For each endogenous cluster $[v]$ with $v \in V \sm U$, clusterwise unique solvability (Assumption~\ref{ass:leu}) provides a cluster solution function $\Phi^{[v]}$ that expresses the endogenous variables $X_{(V \sm U) \cap [v]}$ as a function of the parent variables $X_{\pa_{\vv{G}}([v])}$. 
  Write $\Phi_w^{[v]}$ for the component of $\Phi^{[v]}$ corresponding to variable $w \in (V \sm U) \cap [v]$.

  Replace the original system of equations by the acyclic system: for each endogenous variable $v \in (V \sm U)$, the structural equation is
  \[ X_v = \Phi_v^{[v]}(X_{\pa_{\vv{G}}([v])}). \]
  Each exogenous variable $X_u$ ($u \in U$) retains its original independent distribution. 
  By construction, the solutions of the rewritten system coincide with the solutions of the original system (Proposition~\ref{prop:global-solution}): in both cases, variables are determined by recursively substituting cluster solution functions along the partial ordering of the clusters. 
  The rewritten system is an acyclic SCM: the structural equation for each variable $v$ depends only on variables in strictly earlier clusters.

  Now by Lemma~\ref{lem:b-sep_equiv_d-sep}, the assumption
  \[ A \Perp^{b}_{\vv{G}} B \given C \]
  implies:
  \[ A \Perp^{d}_{\vv{G}^\acy} B \given C. \]
  The graph of the rewritten acyclic SCM is a subgraph of the acyclification DAG $\vv{G}^\acy$ (it can be a strict subgraph if one or more functional dependences cancel out).
  By the standard directed global Markov property for acyclic SCMs with independent exogenous variables \citep[][see also \citealp{Lauritzen1996,Pearl2009}]{LauritzenDawidLarsenLeimer1990}:
  \[ X_A \Indep_{\Prb(X_V)} X_B \given X_C. \qedhere \]
\end{proof}

\subsection{Accounting for Determinism}

We now strengthen the Markov property for partially oriented bipartite graphs by taking into account determinism,
analogous to how $D$-separation \citep{GeigerVermaPearl1990} strengthens $d$-separation in Bayesian networks.

We will write $S \preceq T$ iff $X_S$ is a measurable function of $X_T$.

\begin{lemma}\label{lem:condition_on_functionally_determined}
  Let $A,B,S,T \subseteq V$ with $S \preceq T$. Then:
  \[ X_A \Indep_{\Prb(X_V)} X_B \given X_{S\cup T} \iff X_A \Indep_{\Prb(X_V)} X_B \given X_T. \]
\end{lemma}
\begin{proof}
  This follows from the elementary axioms for conditional independence \citep{Dawid1979}.
\end{proof}

Although $B$-blocking is stated with $\ant_{\vv{G}}(C)$ in its collider rule (Definition~\ref{def:B-blocking}), $B$-separation in fact coincides with $b$-separation given the enlarged conditioning set $\fdet_{\vv{G}}(C)$.
This equivalence is the graphical device behind the strengthening.

\begin{lemma}\label{lem:B-eq-b-fdet}
For all $C \subseteq V$, a walk is $B$-blocked by $C$ if and only if it is $b$-blocked by $\fdet_{\vv{G}}(C)$. Hence, for all $A, B, C \subseteq V$: 
  \[ A \Perp^{B}_{\vv{G}} B \given C \iff A \Perp^{b}_{\vv{G}} B \given \fdet_{\vv{G}}(C). \]
\end{lemma}
\begin{proof}
We show the per-walk equivalence; the separation statement follows by quantifying over all walks (equivalently, all paths).
Running Definition~\ref{def:b-blocking} at $\fdet_{\vv{G}}(C)$, rule~3 is subsumed by rule~2: a two-distinct-exit segment with $\pa_{\vv{G}}([s_{i,1}]) \subseteq \fdet_{\vv{G}}(C)$ has its whole cluster---hence both exits---in $\fdet_{\vv{G}}(C)$. So $b$-blocking by $\fdet_{\vv{G}}(C)$ reduces to: a collider segment not meeting $\ant_{\vv{G}}(\fdet_{\vv{G}}(C))$, or a non-collider segment with an exit in $\fdet_{\vv{G}}(C)$. This differs from $B$-blocking by $C$ only in the collider condition, which uses $\ant_{\vv{G}}(\fdet_{\vv{G}}(C))$ rather than $\ant_{\vv{G}}(C)$.

Since $\ant_{\vv{G}}(C) \subseteq \ant_{\vv{G}}(\fdet_{\vv{G}}(C))$, every walk $b$-blocked by $\fdet_{\vv{G}}(C)$ is $B$-blocked by $C$. Conversely, let $\pi$ be $B$-blocked by $C$; we show it is $b$-blocked by $\fdet_{\vv{G}}(C)$. As the two criteria share the non-collider rule, we may assume the block comes from a collider segment $s^c$, in a cluster $[c^\ast]$, that does not meet $\ant_{\vv{G}}(C)$. If $[c^\ast]$ also fails to meet $\ant_{\vv{G}}(\fdet_{\vv{G}}(C))$, then $s^c$ blocks $\pi$ under $b$-blocking by $\fdet_{\vv{G}}(C)$ as well. Otherwise $[c^\ast]$ meets $\ant_{\vv{G}}(\fdet_{\vv{G}}(C))$ but not $\ant_{\vv{G}}(C)$, and we claim $\pa_{\vv{G}}([c^\ast]) \subseteq \fdet_{\vv{G}}(C)$.

By Lemma~\ref{lem:acyclification_ancestral}, meeting $\ant_{\vv{G}}(S)$ is equivalent to $[c^\ast]$ meeting $\anc_{\vv{G}^\acy}(S)$; and in the acyclification $\anc_{\vv{G}^\acy}(\fdet_{\vv{G}}(C)) = \anc_{\vv{G}^\acy}(C) \cup \fdet_{\vv{G}}(C)$. (For ``$\subseteq$'': along a directed path from an ancestor $v$ to a determined node $w$, either the path meets $C$, so $v \in \anc_{\vv{G}^\acy}(C)$, or, reading it back from $w$, each node is a parent of a determined node outside $C$ and hence itself determined, so $v \in \fdet_{\vv{G}}(C)$.) Thus $[c^\ast]$ meets $\fdet_{\vv{G}}(C)$: pick $y \in [c^\ast] \cap \fdet_{\vv{G}}(C)$. As $[c^\ast]$ avoids $\ant_{\vv{G}}(C) \supseteq C$, we have $y \notin C$, so $y$ is endogenous with $\pa_{\vv{G}}([c^\ast]) = \pa_{\vv{G}}([y]) \subseteq \fdet_{\vv{G}}(C)$, proving the claim.

The two walk-neighbors of $s^c$ are variables of $\pa_{\vv{G}}([c^\ast]) \subseteq \fdet_{\vv{G}}(C)$---they point into the boundary equations of $s^c$---and each is an exit of the adjacent non-collider segment. Hence both neighbors are $b$-blocked by $\fdet_{\vv{G}}(C)$, so $\pi$ is $b$-blocked by $\fdet_{\vv{G}}(C)$.
\end{proof}

Combined with the acyclification, this yields a clean correspondence between $B$-separation and $D$-separation that mirrors Lemma~\ref{lem:b-sep_equiv_d-sep}, now accounting for determinism.
\begin{lemma}[$B$-separation equals $D$-separation in the acyclification]\label{lem:B-sep_equiv_D-sep}
For all $A, B, C \subseteq V$:
\[ A \Perp^{B}_{\vv{G}} B \given C \iff A \Perp^{D}_{\vv{G}^\acy} B \given C. \]
\end{lemma}
\begin{proof}
For every $v \in V$ the acyclification satisfies $\pa_{\vv{G}^\acy}(v) = \pa_{\vv{G}}([v])$, so the closure recursions of Definitions~\ref{def:functionally_determined} and~\ref{def:functionally_determined_DAG} coincide; hence $\fdet_{\vv{G}^\acy}(C) = \fdet_{\vv{G}}(C)$ for all $C \subseteq V$. Therefore
\begin{align*}
A \Perp^{B}_{\vv{G}} B \given C
&\iff A \Perp^{b}_{\vv{G}} B \given \fdet_{\vv{G}}(C) && \text{(Lemma~\ref{lem:B-eq-b-fdet})}\\
&\iff A \Perp^{d}_{\vv{G}^\acy} B \given \fdet_{\vv{G}}(C) && \text{(Lemma~\ref{lem:b-sep_equiv_d-sep})}\\
&\iff A \Perp^{d}_{\vv{G}^\acy} B \given \fdet_{\vv{G}^\acy}(C) && (\fdet_{\vv{G}^\acy}(C) = \fdet_{\vv{G}}(C))\\
&\iff A \Perp^{D}_{\vv{G}^\acy} B \given C && \text{(Definition~\ref{def:D-sep}).}
\end{align*}
\end{proof}

Theorem~\ref{thm:Markov_B} now follows from Theorem~\ref{thm:Markov_b}.
\thmMarkovB*
\begin{proof}
  By Lemma~\ref{lem:B-eq-b-fdet}, the hypothesis $A \Perp^{B}_{\vv{G}} B \given C$ is equivalent to
  \[ A \Perp^{b}_{\vv{G}} B \given \fdet_{\vv{G}}(C). \]
  By Theorem~\ref{thm:Markov_b},
  \[ X_A \Indep_{\Prb(X_V)} X_B \given X_{\fdet_{\vv{G}}(C)}. \]
  The clusterwise solvability condition gives, for each endogenous cluster $[c]$:
  \[ [c] \cap V \preceq \pa_{\vv{G}}([c]) \]
  From the definitions, it follows that $\fdet_{\vv{G}}(C) \preceq C$.
  By Lemma~\ref{lem:condition_on_functionally_determined}, it then suffices to condition only on $C$:
  \[ X_A \Indep_{\Prb(X_V)} X_B \given X_C. \qedhere \]
\end{proof}

Similarly to $b$-separation, also $B$-separation may equivalently be defined via walks or via paths.
\begin{lemma}[$B$-separation via walks or paths]\label{lem:B-sep_walk_path}
For all $A, B, C \subseteq V$, every walk from a node in $A$ to a node in $B$ is $B$-blocked by $C$ if and only if every path from a node in $A$ to a node in $B$ is $B$-blocked by $C$.
\end{lemma}
\begin{proof}
By Lemma~\ref{lem:B-eq-b-fdet} (per-walk form), a walk is $B$-blocked by $C$ iff it is $b$-blocked by $\fdet_{\vv{G}}(C)$; applied to walks and to paths separately, this gives that $B$-separation given $C$, defined via walks or via paths, equals $b$-separation given $\fdet_{\vv{G}}(C)$ defined via walks resp.\ paths. The latter two coincide by Lemma~\ref{lem:b-sep_walk_path}.
\end{proof}

Because $\fdet_{\vv{G}}(C)$ can be strictly larger than $C$, $B$-separation is genuinely stronger than $b$-separation: $A \Perp^{b}_{\vv{G}} B \given C$ implies $A \Perp^{B}_{\vv{G}} B \given C$ (comparing the two criteria at the same $C$: the collider rule is common to both, while each non-collider case of $b$-blocking---an exit in $C$, or two distinct exits whose cluster satisfies $\pa_{\vv{G}}([c]) \subseteq C$---puts an exit in $\fdet_{\vv{G}}(C)$ and hence $B$-blocks), but not conversely, as Example~\ref{ex:B-stronger-b} shows.

\begin{example}[$B$-separation is strictly stronger than $b$-separation]\label{ex:B-stronger-b}
Consider the observational bathtub graph $\vv{G}$ of Example~\ref{ex:bathtub-ordering} and condition on $C=\{X_I\}$. Since $\pa_{\vv{G}}([X_O])=\{X_I\}$, the outflow is functionally determined by $C$, so $\fdet_{\vv{G}}(\{X_I\})=\{X_I,X_O\}$; the variables $X_P$ and $X_D$ are not determined, as their clusters also require $X_K$, resp.\ $X_g$.

Take $A=\{X_O\}$ and $B=\{X_D\}$. The walk $X_O \to f_2 \otto X_P \to f_3 \otto X_D$ splits into the segments $\{X_O\}$, $\{f_2, X_P\}$, and $\{f_3, X_D\}$, with single exits $X_O$, $X_P$, and $X_D$ respectively. None of these exits lies in $C=\{X_I\}$, hence this walk is $b$-open and $X_O \nPerp^{b}_{\vv{G}} X_D \given X_I$. Under $B$-separation the outcome differs: every walk out of $X_O$ begins with the endpoint segment $\{X_O\}$ whose exit $X_O$ lies in $\fdet_{\vv{G}}(\{X_I\})$, so $X_O \Perp^{B}_{\vv{G}} X_D \given X_I$.

The $B$-separation verdict is the correct one: conditioning on $X_I$ forces $X_O = X_I$ (Example~\ref{ex:bathtub-ordering}), so $X_O$ is constant given $X_I$ and hence trivially independent of $X_D$. Thus $B$-separation detects a conditional independence arising from determinism that $b$-separation misses. 
\end{example}

\section{Proof of the Extended Global Markov Property}\label{app:proof-extended}

We prove Theorem~\ref{thm:i-Markov_B}, which extends the Global Markov Property (Theorem~\ref{thm:Markov_B}) from joint distributions to Markov kernels with non-random input variables.
Our strategy will be similar to that of Appendix~\ref{app:direct-proof}, but rather than applying the standard global Markov property for acyclic SCMs, we make use of the global Markov property for causal Bayesian networks
with input variables established by \citet{Forre2021}.

\citet[Theorem~6.3]{Forre2021} proves the global Markov property for CBNs with input variables using \emph{transitional conditional independence} \citep[Definition~3.1]{Forre2021}. 
This is an asymmetric notion of conditional independence for Markov kernels: $X_A \Indep_{\Prb(\cdot \mk X_J)} X_B \given X_C$ means that there \emph{exists} a Markov kernel $Q(X_A \mk X_C)$ (not depending on $X_B$) such that
\[ \Prb(X_A, X_B, X_C \mk X_J) = Q(X_A \mk X_C) \otimes \Prb(X_B, X_C \mk X_J). \]
The proof of \citet{Forre2021} proceeds by induction over the topological ordering of the conditional DAG (CDAG)---which marks the variables in $J \subseteq U \subseteq V$ as \emph{input} variables---chaining the (asymmetric) separoid rules for transitional conditional independence and $d$-separation in CDAGs, an (asymmetric) extension of $d$-separation in DAGs. 
A crucial feature of \citeauthor{Forre2021}'s approach is that it does \emph{not} rely on symmetry of conditional independence (which fails for Markov kernels in general), but instead uses left and right versions of the separoid rules separately.


\begin{theorem}\label{thm:i-Markov_b}
Suppose Assumption~\ref{ass:leu} holds.
Treat exogenous variables $J\subseteq U$ as non-random and assign independent distributions to the remaining exogenous variables in $U\sm J$, yielding the Markov kernel $\Prb(X_V\mk X_J)$.
Then for all $A,B,C\subseteq V$:
\[ A \Perp^{b}_{\vv{G}} B \cup J \given C \implies X_A \Indep_{\Prb(X_V \mk X_J)} X_B \given X_C. \]
\end{theorem}
\begin{proof}
  As in the proof of Theorem~\ref{thm:Markov_b}, solutions of the system satisfy the acyclic system of equations
  \[ X_v = \Phi_v^{[v]}(X_{\pa_{\vv{G}}([v])}), \qquad v \in V \sm U. \]
  Treat $J \subseteq U$ as non-random and put independent distributions on the remaining exogenous variables $W := U \sm J$.
  For each endogenous cluster $[v]$ ($v \in V \sm U$), let the deterministic Markov kernel $\mathcal{X}_{\pa_{\vv{G}}([v])} \to \mathcal{P}(\mathcal{X}_v)$ be the one corresponding to the cluster solution function $\Phi_v^{[v]}$,
  and for each $w \in W$ let $\Prb(X_w)$ be its distribution.
  Together with the non-stochastic inputs $X_J$, these Markov kernels define a \emph{causal Bayesian network} $\mathcal{M}$ in the sense of \citet[Definition~6.1]{Forre2021}, with:
  \begin{itemize}
      \item non-stochastic input variables: $X_J$,
      \item stochastic variables: $X_{V \sm J}$,
      \item graph: the acyclification $\vv{G}^\acy$, viewed as a conditional DAG $\vv{G}^\acy(V \sm J \given \doit(J))$.
  \end{itemize}
  This is a valid conditional DAG: $\vv{G}^\acy$ is acyclic, and no edge points into any $j\in J$, since each such $j$ is an exogenous singleton cluster with $\pa_{\vv{G}}([j])=\emptyset$. 
  The joint Markov kernel of $\mathcal{M}$ is precisely $\Prb(X_V \mk X_J)$.

  Let $A, B, C \subseteq V$ be such that
  \[ A \Perp^{b}_{\vv{G}} B \cup J \given C. \]
  By Lemma~\ref{lem:b-sep_equiv_d-sep},
  \[ A \Perp^{d}_{\vv{G}^\acy} B \cup J \given C. \]
  Since $\vv{G}^\acy$ is acyclic, $d$-separation coincides with $\sigma$-separation \citep[Definition~5.9 and Remark~5.10]{Forre2021}.
  Reading the acyclification as the conditional DAG $\vv{G}^\acy(V \sm J \given \doit(J))$, whose $\sigma$-separation criterion implicitly includes the input nodes $J$ on the right, this is exactly
  \[ A \Perp^{\sigma}_{\vv{G}^\acy(V \sm J \given \doit(J))} B \given C. \]
  By the global Markov property of \citet[Theorem~6.3]{Forre2021} applied to $\mathcal{M}$, this implies the transitional conditional independence
  \[ X_A \Indep_{\Prb(X_V \mk X_J)} X_B \given X_C. \qedhere \]
\end{proof}
\begin{corollary}\label{thm:i-Markov_b_restricted}
  Suppose Assumption~\ref{ass:leu} holds.
  Treat exogenous variables $J \subseteq U$ as non-random, and assign independent distributions to exogenous variables in $U \sm J$, yielding Markov kernel $\Prb(X_V \mk X_J)$.
  Then for all $A, B, C \subseteq V$ such that $J \subseteq B \cup C$:
  \[ A \Perp^{b}_{\vv{G}} B \given C \implies X_A \Indep_{\Prb(X_V \mk X_J)} X_B \given X_C. \]
\end{corollary}
\begin{proof}
When $J\subseteq B\cup C$ we have $J\sm B\subseteq C$, so every walk from $A$ to a node of $J\sm B$ ends in $C$ and is therefore $b$-blocked by $C$; hence
$A\Perp^{b}_{\vv{G}}B\given C$ implies $A\Perp^{b}_{\vv{G}}B\cup J\given C$, and Theorem~\ref{thm:i-Markov_b} yields $X_A\Indep_{\Prb(X_V\mk X_J)}X_B\given X_C$.
\end{proof}

\subsection{Accounting for Determinism}

Just as we strengthened the global Markov property for partially oriented bipartite graphs by accounting for determinism, we can do the same for the extended version.
Definition~\ref{def:functionally_determined} still applies, and does not need to distinguish exogenous input nodes $J$ from exogenous random nodes $U \sm J$: exogenous variable nodes are only functionally determined by $C$ if they are in $C$.

\begin{lemma}\label{lem:condition_on_functionally_determined_Markov_kernel}
  Let $A,B,S,T \subseteq V$ with $S \preceq T$. Then:
  \[ X_A \Indep_{\Prb(X_V \mk X_J)} X_B \given X_{S \cup T} \iff X_A \Indep_{\Prb(X_V \mk X_J)} X_B \given X_T. \]
\end{lemma}
\begin{proof}
  This is the transitional-conditional-independence analog of Lemma~\ref{lem:condition_on_functionally_determined}; because transitional conditional independence is asymmetric, we cannot use the symmetric argument and instead invoke Forr\'e's Equivalent Exchange rule.
  By hypothesis $X_S$ is a measurable function of $X_T$.
  Then $X_{S \cup T}$ is a measurable function of $X_T$; conversely $X_T$ is a coordinate projection of $X_{S \cup T}$.
  Since all model variables are measurable functions of the exogenous variables $X_U$, both $X_{S \cup T}$ and $X_T$ are deterministic transitional random variables; as each is a measurable function of the other, they are equivalent \citep[Notation~2.19 and Remark~2.20]{Forre2021}.
  The Equivalent Exchange rule for transitional conditional independence \citep[Corollary~3.14]{Forre2021}, which allows the conditioning variable to be replaced by an equivalent one, then yields both implications:
  \[ X_A \Indep_{\Prb(X_V \mk X_J)} X_B \given X_{S \cup T} \iff X_A \Indep_{\Prb(X_V \mk X_J)} X_B \given X_T. \qedhere \]
\end{proof}

We obtain Theorem~\ref{thm:i-Markov_B} from Corollary~\ref{thm:i-Markov_b_restricted}.
\thmiMarkovB*
\begin{proof}
  By Lemma~\ref{lem:B-eq-b-fdet}, the hypothesis $A \Perp^{B}_{\vv{G}} B \given C$ is equivalent to
  \[ A \Perp^{b}_{\vv{G}} B \given \fdet_{\vv{G}}(C). \]
  By Corollary~\ref{thm:i-Markov_b_restricted} (whose hypothesis $J \subseteq B \cup \fdet_{\vv{G}}(C)$ holds because $J \subseteq B \cup C \subseteq B \cup \fdet_{\vv{G}}(C)$),
  \[ X_A \Indep_{\Prb(X_V \mk X_J)} X_B \given X_{\fdet_{\vv{G}}(C)}. \]
  The clusterwise solvability condition gives, for each endogenous cluster $[c]$:
  \[ [c] \cap V \preceq \pa_{\vv{G}}([c]). \]
  From the definitions, it follows that $\fdet_{\vv{G}}(C) \preceq C$.
  By Lemma~\ref{lem:condition_on_functionally_determined_Markov_kernel}, it then suffices to condition only on $C$:
  \[ X_A \Indep_{\Prb(X_V \mk X_J)} X_B \given X_C. \qedhere \]
\end{proof}

%
%
%
%
%
%
%
%

\section{Physical Implementations of Hard Interventions}\label{app:implementations}

The formal notation $\doit(f_j : X_v = \xi_v)$ has a natural physical interpretation: it specifies \emph{which mechanism} $f_j$ in the system is replaced in order to enforce $X_v = \xi_v$. Different choices of $f_j$ correspond to genuinely different physical procedures for achieving the same target value. We illustrate this for four of the hard interventions on the bathtub model. Combined with the ones in Example~\ref{ex:bathtub-ambiguous}, this gives a complete ``physical implementation'' of the causal semantics of the bathtub system under the hard interventions we consider elementary in our framework.

\paragraph{$\doit(f_1 : X_O = \xi_O)$.}
Equation $f_1$ (the equilibrium condition $X_I = X_O$) is replaced by $\tilde{f}_1 : X_O = \xi_O$. The causal ordering is preserved: $\tilde{f}_1$ determines $X_O$, $f_2$ determines $X_P$, and $f_3$ determines $X_D$. A physical implementation is to divert the original inflow away from the tub and install a new faucet with inflow rate $X_{I_2} = \xi_O$.

\paragraph{$\doit(f_1 : X_P = \xi_P)$.}
Equation $f_1$ is replaced by $\tilde{f}_1 : X_P = \xi_P$. The causal ordering changes: $\tilde{f}_1$ now determines $X_P$ (instead of $X_O$), and consequently $f_2$ must solve for $X_O$ given $X_P$. A physical implementation requires diverting the original inflow, installing a sufficiently large new inflow, and connecting a pressure relief valve to the bottom of the tub that activates when $X_P > \xi_P$.

\paragraph{$\doit(f_2 : X_P = \xi_P)$.}
Here, Torricelli's law $f_2$ is replaced by $\tilde{f}_2 : X_P = \xi_P$, while the equilibrium condition $f_1$ remains intact. This requires a more involved physical procedure: clog the drain, reroute the original inflow directly to the outflow through a pipe (bypassing the tub and drain), install an additional sufficiently large inflow, and connect a pressure relief valve that activates when $X_P > \xi_P$. Note that this intervention targets the same variable ($X_P = \xi_P$) as $\doit(f_1 : X_P = \xi_P)$, but through a different mechanism: the equilibrium condition $f_1$ is preserved, so $X_O = X_I$ still holds, whereas in $\doit(f_1 : X_P = \xi_P)$ we get $X_O = X_K \sqrt{\xi_P}$.

\paragraph{$\doit(f_2 : X_D = \xi_D)$.}
Torricelli's law $f_2$ is replaced by $\tilde{f}_2 : X_D = \xi_D$. This changes the causal ordering: $f_3$ now determines $X_P$ from $X_D$ (rather than $X_D$ from $X_P$). The physical implementation is similar to the previous case---clog the drain, reroute inflow to outflow, install an additional inflow---but instead of a pressure valve, the bathtub is cut at height $\xi_D$.

These examples illustrate a key advantage of the BGCM framework: the notion $\doit(f_j : X_v = \xi_v)$ makes the physical implementation explicit by specifying which mechanism is targeted, resolving the ambiguity inherent in the standard notion $\doit(X_v = \xi_v)$.
Furthermore, we have shown explicitly that each such hard intervention that leads to a solvable system can indeed be realized as a real-world intervention.

\section{Partially Oriented Bipartite Graphs under Interventions}\label{app:intervention-graphs}

Figure~\ref{fig:intervention-graphs} shows the partially oriented bipartite graphs for the observational bathtub model, all six well-defined hard interventions (cf.\ Table~\ref{tab:interventions}), and the three infeasible hard interventions. Edges that change relative to the observational case are drawn in red; intervened equation nodes are also shown in red. The systems that (generically) do not have solutions are drawn in gray.

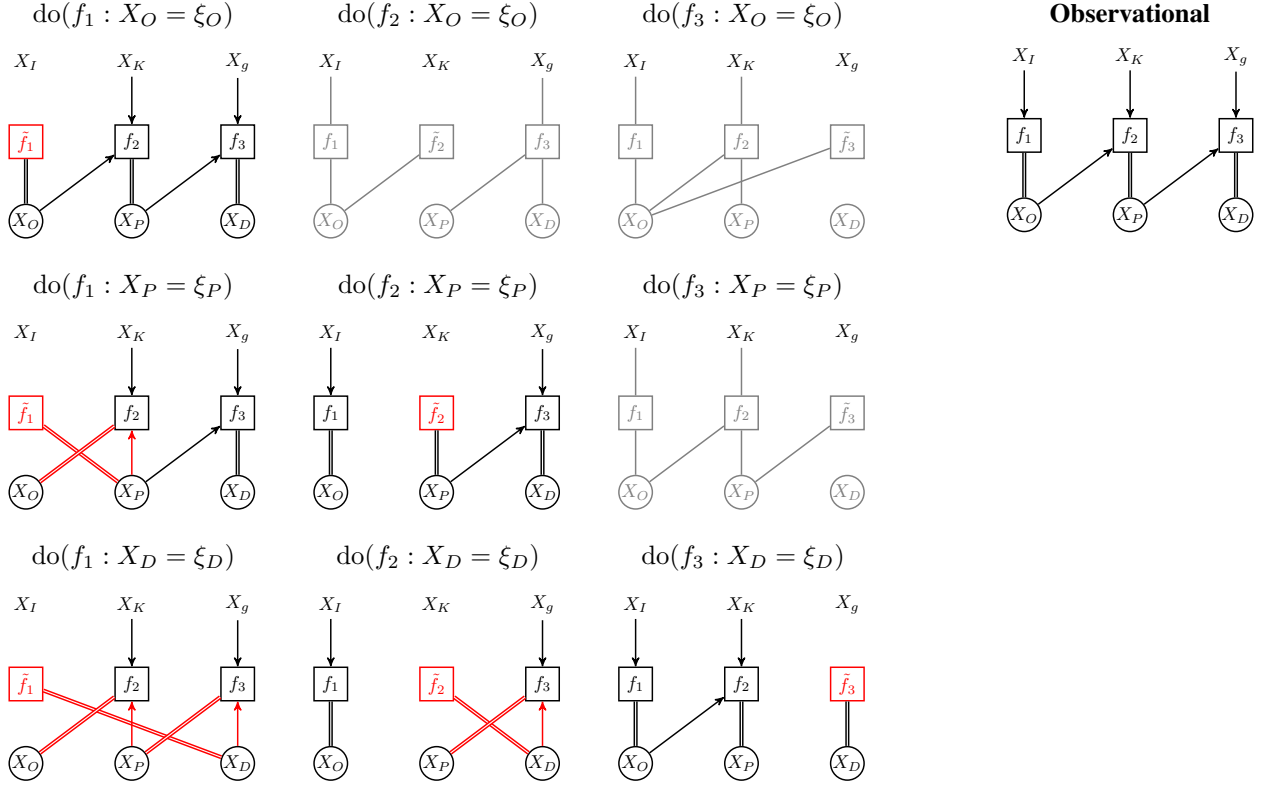
\begin{figure*}[t]
\centering

\begin{minipage}[t]{0.23\textwidth}
\centering\textbf{$\doit(f_1 : X_O = \xi_O)$}\\[6pt]
\scalebox{0.7}{\begin{tikzpicture}
\node[var] (vO) at (2,0) {$X_O$};
\node[var] (vP) at (4,0) {$X_P$};
\node[var] (vD) at (6,0) {$X_D$};
\node[varc,red] (f1) at (2,1.5) {\color{red}$\tilde{f}_1$};
\node[varc] (f2) at (4,1.5) {$f_2$};
\node[varc] (f3) at (6,1.5) {$f_3$};
\node (XI) at (2,3) {$X_I$};
\node (XK) at (4,3) {$X_K$};
\node (Xg) at (6,3) {$X_g$};
\draw[otto] (f1) -- (vO);
\draw[arr] (vO) -- (f2); \draw[otto] (f2) -- (vP); \draw[arr] (XK) -- (f2);
\draw[arr] (vP) -- (f3); \draw[arr] (Xg) -- (f3); \draw[otto] (f3) -- (vD);
\end{tikzpicture}}
\end{minipage}
\begin{minipage}[t]{0.23\textwidth}
\centering\textbf{$\doit(f_2 : X_O = \xi_O)$}\\[6pt]
\scalebox{0.7}{\begin{tikzpicture}
\node[var,gray] (vO) at (2,0) {$X_O$};
\node[var,gray] (vP) at (4,0) {$X_P$};
\node[var,gray] (vD) at (6,0) {$X_D$};
\node[varc,gray] (f1) at (2,1.5) {$f_1$};
\node[varc,gray] (f2) at (4,1.5) {$\tilde{f}_2$};
\node[varc,gray] (f3) at (6,1.5) {$f_3$};
\node (XI) at (2,3) {$X_I$};
\node (XK) at (4,3) {$X_K$};
\node (Xg) at (6,3) {$X_g$};
\draw[noarr,gray] (XI) -- (f1); \draw[noarr,gray] (f1) -- (vO);
\draw[noarr,gray] (vO) -- (f2); 
\draw[noarr,gray] (vP) -- (f3); \draw[noarr,gray] (Xg) -- (f3); \draw[noarr,gray] (f3) -- (vD);
\end{tikzpicture}}
\end{minipage}
\begin{minipage}[t]{0.23\textwidth}
\centering\textbf{$\doit(f_3 : X_O = \xi_O)$}\\[6pt]
\scalebox{0.7}{\begin{tikzpicture}
\node[var,gray] (vO) at (2,0) {$X_O$};
\node[var,gray] (vP) at (4,0) {$X_P$};
\node[var,gray] (vD) at (6,0) {$X_D$};
\node[varc,gray] (f1) at (2,1.5) {$f_1$};
\node[varc,gray] (f2) at (4,1.5) {$f_2$};
\node[varc,gray] (f3) at (6,1.5) {$\tilde{f}_3$};
\node (XI) at (2,3) {$X_I$};
\node (XK) at (4,3) {$X_K$};
\node (Xg) at (6,3) {$X_g$};
\draw[noarr,gray] (XI) -- (f1); \draw[noarr,gray] (f1) -- (vO);
\draw[noarr,gray] (vO) -- (f2); \draw[noarr,gray] (f2) -- (vP); \draw[noarr,gray] (XK) -- (f2);
\draw[noarr,gray] (f3) -- (vO);
\end{tikzpicture}}
\end{minipage}\hfill
\begin{minipage}[t]{0.23\textwidth}
\centering\textbf{Observational}\\[6pt]
\scalebox{0.7}{\begin{tikzpicture}
\node[var] (vO) at (2,0) {$X_O$};
\node[var] (vP) at (4,0) {$X_P$};
\node[var] (vD) at (6,0) {$X_D$};
\node[varc] (f1) at (2,1.5) {$f_1$};
\node[varc] (f2) at (4,1.5) {$f_2$};
\node[varc] (f3) at (6,1.5) {$f_3$};
\node (XI) at (2,3) {$X_I$};
\node (XK) at (4,3) {$X_K$};
\node (Xg) at (6,3) {$X_g$};
\draw[arr] (XI) -- (f1); \draw[otto] (f1) -- (vO);
\draw[arr] (vO) -- (f2); \draw[otto] (f2) -- (vP); \draw[arr] (XK) -- (f2);
\draw[arr] (vP) -- (f3); \draw[arr] (Xg) -- (f3); \draw[otto] (f3) -- (vD);
\end{tikzpicture}}
\end{minipage}
  
\bigskip

\begin{minipage}[t]{0.23\textwidth}
\centering\textbf{$\doit(f_1 : X_P = \xi_P)$}\\[6pt]
\scalebox{0.7}{\begin{tikzpicture}
\node[var] (vO) at (2,0) {$X_O$};
\node[var] (vP) at (4,0) {$X_P$};
\node[var] (vD) at (6,0) {$X_D$};
\node[varc,red] (f1) at (2,1.5) {\color{red}$\tilde{f}_1$};
\node[varc] (f2) at (4,1.5) {$f_2$};
\node[varc] (f3) at (6,1.5) {$f_3$};
\node (XI) at (2,3) {$X_I$};
\node (XK) at (4,3) {$X_K$};
\node (Xg) at (6,3) {$X_g$};
\draw[otto,red] (f1) -- (vP);
\draw[otto,red] (f2) -- (vO); \draw[arr,red] (vP) -- (f2); \draw[arr] (XK) -- (f2);
\draw[arr] (vP) -- (f3); \draw[arr] (Xg) -- (f3); \draw[otto] (f3) -- (vD);
\end{tikzpicture}}
\end{minipage}
\begin{minipage}[t]{0.23\textwidth}
\centering\textbf{$\doit(f_2 : X_P = \xi_P)$}\\[6pt]
\scalebox{0.7}{\begin{tikzpicture}
\node[var] (vO) at (2,0) {$X_O$};
\node[var] (vP) at (4,0) {$X_P$};
\node[var] (vD) at (6,0) {$X_D$};
\node[varc] (f1) at (2,1.5) {$f_1$};
\node[varc,red] (f2) at (4,1.5) {\color{red}$\tilde{f}_2$};
\node[varc] (f3) at (6,1.5) {$f_3$};
\node (XI) at (2,3) {$X_I$};
\node (XK) at (4,3) {$X_K$};
\node (Xg) at (6,3) {$X_g$};
\draw[arr] (XI) -- (f1); \draw[otto] (f1) -- (vO);
\draw[otto] (f2) -- (vP);
\draw[arr] (vP) -- (f3); \draw[arr] (Xg) -- (f3); \draw[otto] (f3) -- (vD);
\end{tikzpicture}}
\end{minipage}
\begin{minipage}[t]{0.23\textwidth}
\centering\textbf{$\doit(f_3 : X_P = \xi_P)$}\\[6pt]
\scalebox{0.7}{\begin{tikzpicture}
\node[var,gray] (vO) at (2,0) {$X_O$};
\node[var,gray] (vP) at (4,0) {$X_P$};
\node[var,gray] (vD) at (6,0) {$X_D$};
\node[varc,gray] (f1) at (2,1.5) {$f_1$};
\node[varc,gray] (f2) at (4,1.5) {$f_2$};
\node[varc,gray] (f3) at (6,1.5) {$\tilde{f}_3$};
\node (XI) at (2,3) {$X_I$};
\node (XK) at (4,3) {$X_K$};
\node (Xg) at (6,3) {$X_g$};
\draw[noarr,gray] (XI) -- (f1); \draw[noarr,gray] (f1) -- (vO);
\draw[noarr,gray] (vO) -- (f2); \draw[noarr,gray] (f2) -- (vP); \draw[noarr,gray] (XK) -- (f2);
\draw[noarr,gray] (vP) -- (f3); 
\end{tikzpicture}}
\end{minipage}\hfill
\begin{minipage}[t]{0.23\textwidth}
\end{minipage}

\bigskip

\begin{minipage}[t]{0.23\textwidth}
\centering\textbf{$\doit(f_1 : X_D = \xi_D)$}\\[6pt]
\scalebox{0.7}{\begin{tikzpicture}
\node[var] (vO) at (2,0) {$X_O$};
\node[var] (vP) at (4,0) {$X_P$};
\node[var] (vD) at (6,0) {$X_D$};
\node[varc,red] (f1) at (2,1.5) {\color{red}$\tilde{f}_1$};
\node[varc] (f2) at (4,1.5) {$f_2$};
\node[varc] (f3) at (6,1.5) {$f_3$};
\node (XI) at (2,3) {$X_I$};
\node (XK) at (4,3) {$X_K$};
\node (Xg) at (6,3) {$X_g$};
\draw[otto,red] (f1) -- (vD);
\draw[otto,red] (f2) -- (vO); \draw[arr,red] (vP) -- (f2); \draw[arr] (XK) -- (f2);
\draw[otto,red] (f3) -- (vP); \draw[arr] (Xg) -- (f3); \draw[arr,red] (vD) -- (f3);
\end{tikzpicture}}
\end{minipage}
\begin{minipage}[t]{0.23\textwidth}
\centering\textbf{$\doit(f_2 : X_D = \xi_D)$}\\[6pt]
\scalebox{0.7}{\begin{tikzpicture}
\node[var] (vO) at (2,0) {$X_O$};
\node[var] (vP) at (4,0) {$X_P$};
\node[var] (vD) at (6,0) {$X_D$};
\node[varc] (f1) at (2,1.5) {$f_1$};
\node[varc,red] (f2) at (4,1.5) {\color{red}$\tilde{f}_2$};
\node[varc] (f3) at (6,1.5) {$f_3$};
\node (XI) at (2,3) {$X_I$};
\node (XK) at (4,3) {$X_K$};
\node (Xg) at (6,3) {$X_g$};
\draw[arr] (XI) -- (f1); \draw[otto] (f1) -- (vO);
                         \draw[otto,red] (f2) -- (vD);
\draw[otto,red] (f3) -- (vP); \draw[arr] (Xg) -- (f3); \draw[arr,red] (vD) -- (f3);
\end{tikzpicture}}
\end{minipage}
\begin{minipage}[t]{0.23\textwidth}
\centering\textbf{$\doit(f_3 : X_D = \xi_D)$}\\[6pt]
\scalebox{0.7}{\begin{tikzpicture}
\node[var] (vO) at (2,0) {$X_O$};
\node[var] (vP) at (4,0) {$X_P$};
\node[var] (vD) at (6,0) {$X_D$};
\node[varc] (f1) at (2,1.5) {$f_1$};
\node[varc] (f2) at (4,1.5) {$f_2$};
\node[varc,red] (f3) at (6,1.5) {\color{red}$\tilde{f}_3$};
\node (XI) at (2,3) {$X_I$};
\node (XK) at (4,3) {$X_K$};
\node (Xg) at (6,3) {$X_g$};
\draw[arr] (XI) -- (f1); \draw[otto] (f1) -- (vO);
\draw[arr] (vO) -- (f2); \draw[otto] (f2) -- (vP); \draw[arr] (XK) -- (f2);
\draw[otto] (f3) -- (vD);
\end{tikzpicture}}
\end{minipage}\hfill
\begin{minipage}[t]{0.1\textwidth}
\end{minipage}

\caption{Partially oriented bipartite graphs for the bathtub model under all hard interventions (cf.\ Table~\ref{tab:interventions}). The observational graph is shown top-right for reference. Intervened equation nodes are shown in red; edges whose orientation changes relative to the observational case are drawn in red. Note how replacing different equations can lead to fundamentally different causal orderings: e.g., $\doit(f_1 : X_D = \xi_D)$ reverses the causal flow entirely. For three hard interventions, the intervened bipartite graph cannot be oriented (no perfect matching exists), so the undirected intervened bipartite graph is displayed instead.}
\label{fig:intervention-graphs}
\end{figure*}

\section{Partially Oriented Bipartite Graphs for the Domain Invariance Examples}\label{app:domain-graphs}

Figure~\ref{fig:domain-graphs} shows the partially oriented bipartite graphs $\vv{G^R}$ of the joint models constructed for each of the four domain invariance examples in Section~\ref{sec:domain-invariances}. In each case, the exogenous domain indicator variable $R$ is connected to the equation(s) that differ between domains. The graphical structure of $\vv{G^R}$ determines which $B$-separation statements hold, and hence which domain invariances can be derived from the Markov property.

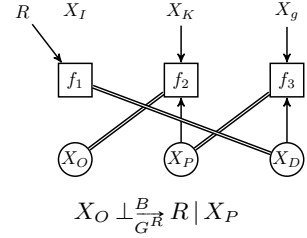
\begin{figure*}[t]
\centering

\begin{minipage}[t]{0.48\textwidth}
\centering\textbf{Example~\ref{eg:example-I}:} obs.\ vs.\ $\doit(X_g = \xi_g)$\\[6pt]
\scalebox{0.7}{\begin{tikzpicture}
\node[var] (vO) at (2,0) {$X_O$};
\node[var] (vP) at (4,0) {$X_P$};
\node[var] (vD) at (6,0) {$X_D$};
\node[var] (vg) at (8,0) {$X_g$};
\node[varc] (f1) at (2,1.5) {$f_1$};
\node[varc] (f2) at (4,1.5) {$f_2$};
\node[varc] (f3) at (6,1.5) {$f_3$};
\node[varc] (f4) at (8,1.5) {$f_4$};
\node (XI) at (2,2.8) {$X_I$};
\node (XK) at (4,2.8) {$X_K$};
\node (Ug) at (6,2.8) {$U_g$};
\node (R) at (8,2.8) {$R$};
\draw[arr] (XI) -- (f1); \draw[otto] (f1) -- (vO);
\draw[arr] (vO) -- (f2); \draw[otto] (f2) -- (vP); \draw[arr] (XK) -- (f2);
\draw[arr] (vP) -- (f3); \draw[arr] (vg) -- (f3); \draw[otto] (f3) -- (vD);
\draw[otto] (f4) -- (vg); \draw[arr] (Ug) -- (f4); \draw[arr] (R) -- (f4);
\end{tikzpicture}}
\\[4pt]
{\small $X_P, X_O \Perp^{B}_{\vv{G^R}} R$}
\end{minipage}\hfill
\begin{minipage}[t]{0.48\textwidth}
\centering\textbf{Example~\ref{eg:example-II}:} obs.\ vs.\ $\doit(f_3 : X_D = \xi_D)$\\[6pt]
\scalebox{0.7}{\begin{tikzpicture}
\node[var] (vO) at (2,0) {$X_O$};
\node[var] (vP) at (4,0) {$X_P$};
\node[var] (vD) at (6,0) {$X_D$};
\node[varc] (f1) at (2,1.5) {$f_1$};
\node[varc] (f2) at (4,1.5) {$f_2$};
\node[varc] (f3) at (6,1.5) {$f_3$};
\node (XI) at (2,2.8) {$X_I$};
\node (XK) at (4,2.8) {$X_K$};
\node (Xg) at (6,2.8) {$X_g$};
\node (R) at (7,2.8) {$R$};
\draw[arr] (XI) -- (f1); \draw[otto] (f1) -- (vO);
\draw[arr] (vO) -- (f2); \draw[otto] (f2) -- (vP); \draw[arr] (XK) -- (f2);
\draw[arr] (vP) -- (f3); \draw[arr] (Xg) -- (f3); \draw[otto] (f3) -- (vD);
\draw[arr] (R) -- (f3);
\end{tikzpicture}}
\\[4pt]
{\small $X_O \Perp^{B}_{\vv{G^R}} R \given X_D, X_P$}
\end{minipage}

\bigskip

\begin{minipage}[t]{0.48\textwidth}
\centering\textbf{Example~\ref{eg:example-III}:} obs.\ vs.\ $\doit(f_1 : X_D = \xi_D)$\\[6pt]
\scalebox{0.7}{\begin{tikzpicture}
\node[var] (vO) at (2,0) {$X_O$};
\node[var] (vP) at (4,0) {$X_P$};
\node[var] (vD) at (6,0) {$X_D$};
\node[varc] (f1) at (2,1.5) {$f_1$};
\node[varc] (f2) at (4,1.5) {$f_2$};
\node[varc] (f3) at (6,1.5) {$f_3$};
\node (XI) at (2,2.8) {$X_I$};
\node (XK) at (4,2.8) {$X_K$};
\node (Xg) at (6,2.8) {$X_g$};
\node (R) at (1,2.8) {$R$};
\draw[otto] (f1) -- (vO);
\draw[otto] (vO) -- (f2); \draw[otto] (f2) -- (vP); \draw[arr] (XK) -- (f2);
\draw[otto] (vP) -- (f3); \draw[arr] (Xg) -- (f3); \draw[otto] (f3) -- (vD);
\draw[otto] (f1) -- (vD);
\draw[arr] (XI) -- (f1); \draw[arr] (R) -- (f1);
\end{tikzpicture}}
\\[4pt]
{\small No non-trivial $B$-separation involving $R$}
\end{minipage}\hfill
\begin{minipage}[t]{0.48\textwidth}
\centering\textbf{Example~\ref{eg:example-IV}:} $\doit(f_1 : X_D\!=\!\xi_D)$ vs.\ $\doit(f_1 : X_D\!=\!\xi_D')$\\[6pt]
\scalebox{0.7}{\begin{tikzpicture}
\node[var] (vO) at (2,0) {$X_O$};
\node[var] (vP) at (4,0) {$X_P$};
\node[var] (vD) at (6,0) {$X_D$};
\node[varc] (f1) at (2,1.5) {$f_1$};
\node[varc] (f2) at (4,1.5) {$f_2$};
\node[varc] (f3) at (6,1.5) {$f_3$};
\node (XI) at (2,2.8) {$X_I$};
\node (XK) at (4,2.8) {$X_K$};
\node (Xg) at (6,2.8) {$X_g$};
\node (R) at (1,2.8) {$R$};
\draw[arr] (vP) -- (f2); \draw[otto] (vO) -- (f2); \draw[arr] (XK) -- (f2);
\draw[arr] (vD) -- (f3); \draw[otto] (vP) -- (f3); \draw[arr] (Xg) -- (f3);
\draw[otto] (f1) -- (vD);
\draw[arr] (R) -- (f1);
\end{tikzpicture}}
\\[4pt]
{\small $X_O \Perp^{B}_{\vv{G^R}} R \given X_P$}
\end{minipage}

\caption{Partially oriented bipartite graphs of the joint models for the four domain invariance examples (Section~\ref{sec:domain-invariances}). In each graph, the exogenous domain indicator $R$ is connected to the equation(s) that differ between domains. The $B$-separation statement below each graph summarizes the graphical criterion from which the corresponding domain invariance is derived via the Markov property. In Example~\ref{eg:example-III}, all endogenous nodes and equation nodes form a single cluster, so no non-trivial $B$-separation involving $R$ exists.}
\label{fig:domain-graphs}
\end{figure*}

\section{A Worked Example with a Genuine Cycle: Supply and Demand}\label{app:supply-demand}

The bathtub of Example~\ref{ex:bathtub-ordering}, although a feedback system at
equilibrium, has a partial orientation whose endogenous clusters each consist of
one equation and one variable ($\{f_1,X_O\}$, $\{f_2,X_P\}$, $\{f_3,X_D\}$); its
causal ordering is therefore acyclic. To illustrate the part of the framework that deals with
genuine cycles---multi-node clusters, whose segments are treated in $B$-separation as
single indivisible units (Definition~\ref{def:B-blocking})---we work out a classic
simultaneous supply--demand system, in which all three mechanisms must be solved
jointly.

Consider a competitive market at equilibrium with endogenous variables
$X_S$ (quantity supplied), $X_D$ (quantity demanded), and $X_P$ (price), and
exogenous supply/demand shifts $X_{U_S}, X_{U_D}$. With supply slope $\beta$ and
demand slope $\alpha$ satisfying $\beta > 0 > \alpha$, the equilibrium is
described by three mechanisms:
\begin{align*}
  f_1: &\quad 0 = X_S - X_D                  &&\text{(market clears)}\\
  f_2: &\quad 0 = \beta X_P + X_{U_S} - X_S   &&\text{(supply)}\\
  f_3: &\quad 0 = \alpha X_P + X_{U_D} - X_D  &&\text{(demand)}
\end{align*}

\paragraph{A single endogenous cluster.}
The endogenous subgraph (variables $\{X_S,X_P,X_D\}$, equations
$\{f_1,f_2,f_3\}$) admits the perfect matching
$M = \{f_1\ttt X_S,\ f_2\ttt X_P,\ f_3\ttt X_D\}$. The closed $M$-alternating
walk
\[
  X_S \ttt f_1 \ttt X_D \ttt f_3 \ttt X_P \ttt f_2 \ttt X_S
\]
(alternating the matched edges $f_1\ttt X_S$, $f_3\ttt X_D$, $f_2\ttt X_P$ with
the unmatched edges $f_1\ttt X_D$, $f_3\ttt X_P$, $f_2\ttt X_S$) visits all six
nodes of the endogenous subgraph. By Definition~\ref{def:equivalence} they therefore form a
single cluster
$\{f_1,f_2,f_3,X_S,X_P,X_D\}$, and by Lemma~\ref{lem:DM} this is
independent of the chosen matching. The cluster has parents
$\pa_{\vv{G}}([f_1]) = \{X_{U_S}, X_{U_D}\}$. In the partial orientation
$\vv{G}$, all six intra-cluster edges are double-undirected, while the two
exogenous shifts point in (Figure~\ref{fig:supply-demand}a). Unlike the bathtub,
the causal ordering here is not acyclic: the entire endogenous system is a single
feedback cluster, solved simultaneously.

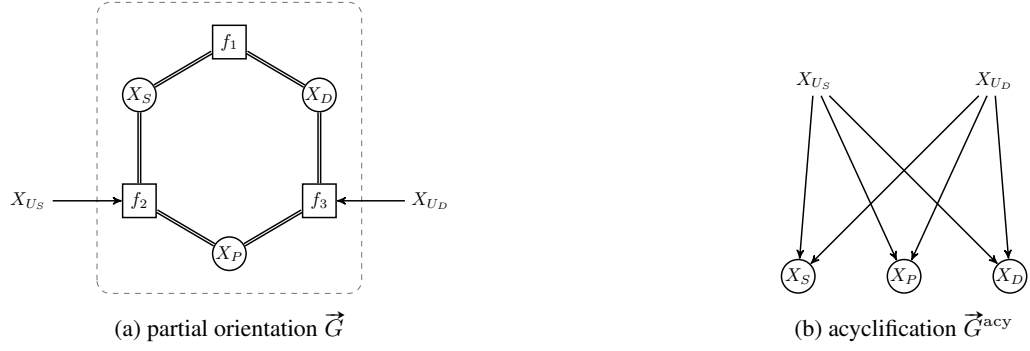
\begin{figure}[t]
  \centering
  \begin{minipage}[b]{0.48\linewidth}
  \centering
  \scalebox{0.7}{\begin{tikzpicture}[x=1cm,y=1cm]
  \draw[dashed, rounded corners=7pt, black!55] (-2.5,-2.75) rectangle (2.5,2.75);
  \node[varc] (f1) at (0,2)     {$f_1$};
  \node[var]  (vS) at (-1.7,1)  {$X_S$};
  \node[varc] (f2) at (-1.7,-1) {$f_2$};
  \node[var]  (vP) at (0,-2)    {$X_P$};
  \node[varc] (f3) at (1.7,-1)  {$f_3$};
  \node[var]  (vD) at (1.7,1)   {$X_D$};
  \node (uS) at (-3.8,-1) {$X_{U_S}$};
  \node (uD) at (3.8,-1)  {$X_{U_D}$};
  \draw[otto] (f1) -- (vS);  \draw[otto] (vS) -- (f2);
  \draw[otto] (f2) -- (vP);  \draw[otto] (vP) -- (f3);
  \draw[otto] (f3) -- (vD);  \draw[otto] (vD) -- (f1);
  \draw[arr] (uS) -- (f2);   \draw[arr] (uD) -- (f3);
  \end{tikzpicture}}\\[4pt]
  {\small (a) partial orientation $\vv{G}$}
  \end{minipage}\hfill
  \begin{minipage}[b]{0.48\linewidth}
  \centering
  \scalebox{0.7}{\begin{tikzpicture}[x=1cm,y=1cm]
  \node (uS) at (-1.7,2.2) {$X_{U_S}$};
  \node (uD) at (1.7,2.2)  {$X_{U_D}$};
  \node[var] (vS) at (-2,-1.5) {$X_S$};
  \node[var] (vP) at (0,-1.5)  {$X_P$};
  \node[var] (vD) at (2,-1.5)  {$X_D$};
  \draw[arr] (uS) -- (vS); \draw[arr] (uS) -- (vP); \draw[arr] (uS) -- (vD);
  \draw[arr] (uD) -- (vS); \draw[arr] (uD) -- (vP); \draw[arr] (uD) -- (vD);
  \end{tikzpicture}}\\[4pt]
  {\small (b) acyclification $\vv{G}^\acy$}
  \end{minipage}
  \caption{The supply--demand system and its acyclification.
  \textbf{(a)}~The partial orientation $\protect\vv{G}$: the three endogenous
  variables and three equations form a single feedback cluster (dashed box) with
  the exogenous shifts $X_{U_S}, X_{U_D}$ as its parents; all intra-cluster edges
  are double-undirected. \textbf{(b)}~The acyclification $\protect\vv{G}^\acy$, the
  DAG on the variable nodes with an edge $v \to v'$ whenever
  $v \in \pa_{\protect\vv{G}}([v'])$; the two shifts become common parents of all
  three endogenous variables, which are mutually nonadjacent. $B$-separation in
  $\protect\vv{G}$ coincides with $D$-separation in $\protect\vv{G}^\acy$
  (Lemma~\ref{lem:B-sep_equiv_D-sep}), so the two encode the same conditional
  independences; but only $\protect\vv{G}$, which retains the equation nodes, can
  model interventions.\label{fig:supply-demand}}
\end{figure}

\paragraph{Unique solvability and solution.}
Because $\beta > 0 > \alpha$, the cluster is
uniquely solvable (Assumption~\ref{ass:leu}). Substituting the supply and demand
relations into the market-clearing condition $X_S = X_D$ gives
$\beta X_P + X_{U_S} = \alpha X_P + X_{U_D}$, hence
\[
  X_P = \frac{X_{U_D} - X_{U_S}}{\beta - \alpha}, \qquad
  X_S = X_D = \frac{\beta X_{U_D} - \alpha X_{U_S}}{\beta - \alpha}.
\]

\paragraph{Reading off (in)dependences.}
Assign independent distributions to the shifts $X_{U_S}, X_{U_D}$, so that
$\Prb(X_V)$ is well defined and the Markov property (Theorem~\ref{thm:Markov_B})
applies. Consider the walk
\[
  X_{U_S} \to f_2 \otto X_P \otto f_3 \ot X_{U_D}.
\]
Its middle segment $f_2 \otto X_P \otto f_3$ lies entirely inside the cluster and
is a \emph{collider segment}: both bounding edges point inward. For
$C = \emptyset$ we have $\ant_{\vv{G}}(\emptyset) = \emptyset$, so this collider
segment is disjoint from $\ant_{\vv{G}}(\emptyset)$ and blocks (rule~1 of
Definition~\ref{def:B-blocking}). Since $X_{U_S}$ and $X_{U_D}$ attach to the
cluster only through the inward edges $X_{U_S}\to f_2$ and $X_{U_D}\to f_3$,
\emph{every} path between them crosses such a within-cluster collider segment,
so
\[
  X_{U_S} \Perp^{B}_{\vv{G}} X_{U_D} \given \emptyset
  \quad\Longrightarrow\quad
  X_{U_S} \Indep_{\Prb(X_V)} X_{U_D},
\]
recovering the assumed independence of the two shifts.

Conditioning on the equilibrium price reverses the verdict. As $X_P$ lies in the
cluster, the whole cluster is anterior to $X_P$, so the collider segment now
meets $\ant_{\vv{G}}(\{X_P\})$ and no longer blocks. Moreover
$\fdet_{\vv{G}}(\{X_P\}) = \{X_P\}$: determining the cluster would require
\emph{both} exogenous parents, and neither is in $\{X_P\}$, so no further node is
functionally determined. Hence the endpoints $X_{U_S}, X_{U_D}$ are not in
$\fdet_{\vv{G}}(\{X_P\})$, rule~2 does not block either, and the walk is
$B$-open:
\[
  X_{U_S} \nPerp^{B}_{\vv{G}} X_{U_D} \given X_P.
\]
This matches the algebra: $X_P = (X_{U_D} - X_{U_S})/(\beta - \alpha)$, so
conditioning on the price imposes the constraint
$X_{U_D} - X_{U_S} = (\beta - \alpha) X_P$ and generically renders the two shifts
dependent.
It is the familiar phenomenon of conditioning on a common effect---except that
here the common effect is an entire feedback cluster rather than a single
variable, which is exactly what rule~1 of $B$-separation is designed to capture.

\paragraph{Determinism (rule~2) in a cycle.}
Rule~2 blocks a non-collider segment as soon as one of its exits is
\emph{functionally determined} by the conditioning set. In a single
multi-variable cluster this is all-or-nothing: every endogenous variable has the
same parents $\{X_{U_S}, X_{U_D}\}$, so none is determined until \emph{both}
shifts are conditioned on. Thus $\fdet_{\vv{G}}(\{X_{U_S}\}) = \{X_{U_S}\}$,
whereas
\[
  \fdet_{\vv{G}}(\{X_{U_S}, X_{U_D}\}) = \{X_{U_S}, X_{U_D}, X_S, X_P, X_D\}
\]
is the entire set of variable nodes. In the latter case the endpoint segment of every path
already has an exit in $\fdet_{\vv{G}}(C)$, so rule~2 blocks it; hence, for
example, $X_S \Perp^{B}_{\vv{G}} X_P \given X_{U_S}, X_{U_D}$. Whereas
conditioning on the endogenous price \emph{created} dependence between the shifts
(rule~1), conditioning on both exogenous shifts \emph{removes} all dependence
among the endogenous quantities: fixing both shifts pins down the equilibrium, so
every endogenous quantity is constant and hence conditionally independent of the
rest. Needing \emph{all} of a cluster's parents before any of its variables is
determined is not special to cycles; it is a general property of $\fdet_{\vv{G}}$
that already appears at single-variable clusters whose equation has several
parents. In the bathtub, for instance, the cluster $\{f_2, X_P\}$ has parents
$\{X_O, X_K\}$, so pressure is determined only once both outflow and drain area
are fixed---conditioning on $X_O$ alone leaves $X_P$ undetermined.

\paragraph{The same facts by $D$-separation.}
By Lemma~\ref{lem:B-sep_equiv_D-sep}, $B$-separation in $\vv{G}$ coincides with
$D$-separation in the acyclification $\vv{G}^\acy$
(Figure~\ref{fig:supply-demand}b), so the three verdicts above can be read off
there as well. In $\vv{G}^\acy$ the shifts $X_{U_S}, X_{U_D}$ are common parents
of $X_S, X_P, X_D$, which are otherwise nonadjacent:
(i)~unconditionally the shifts are $d$-separated, their only connections passing
through the unconditioned colliders $X_S, X_P, X_D$;
(ii)~conditioning on $X_P$ opens that collider, making the shifts $d$-connected;
and (iii)~conditioning on both shifts blocks each path between $X_S$ and $X_P$ path at its fork
($X_{U_S}$ or $X_{U_D}$), giving $d$-separation. The within-cluster collider
segment of $\vv{G}$ thus becomes an ordinary collider at a common child in
$\vv{G}^\acy$, and conditioning on a cluster variable becomes conditioning on that
child. In this example $D$-separation reduces to ordinary $d$-separation, since
conditioning on the parents that determine a variable already blocks the forks
through them.

\end{document}